\documentclass[twoside,11pt]{article}

\usepackage{macros}

\usepackage{jmlr2e}
\usepackage{todonotes}
\usepackage{amsmath}
\usepackage{algpseudocode}
\usepackage{algorithm}
\usepackage{bbm}
\usepackage{todonotes}
\usepackage{enumerate}

\usepackage{wrapfig}

\mathchardef\mhyphen="2D 
\newcommand{\rankelim}{{\tt Rank1Elim}}
\newcommand{\rankelimKL}{{\tt Rank1ElimKL}}
\newcommand{\osub}{{\tt OSUB}}
\newcommand{\uts}{{\tt UTS}}
\newcommand{\ossb}{{\tt OSSB}}

\newcommand{\ucb}{{\tt UCB}}
\newcommand{\klucb}{{\tt KL\mhyphen UCB}}

\ShortHeadings{Solving Bernoulli Rank-One Bandits}{Trinh et al.}
\firstpageno{1}

\begin{document}

\title{Solving Bernoulli Rank-One Bandits with \\ Unimodal Thompson Sampling}

\author{\name Cindy Trinh \email cindy.trinh.sridykhan@gmail.com \\
        \addr ENS Paris-Saclay
        \AND
        \name Emilie Kaufmann \email emilie.kaufmann@univ-lille.fr \\
        \addr CNRS, Universit\'e de Lille, Inria SequeL
        \AND
        \name Claire Vernade \email vernade@google.com \\
        \addr DeepMind, London
        \AND
        \name Richard Combes \email richard.combes@supelec.fr  \\
        \addr CentraleSup\'elec, Gif-sur-Yvette}

\editor{}

\maketitle

\begin{abstract}
\emph{Stochastic Rank-One Bandits} \cite{rank1Ber,stochasticRank1} are a simple framework for regret minimization problems over rank-one matrices of arms. The initially proposed algorithms are proved to have logarithmic regret, but do not match the existing lower bound for this problem. We close this gap by first proving that rank-one bandits are a particular instance of unimodal bandits, and then providing a new analysis of  Unimodal Thompson Sampling (UTS), initially proposed by \cite{UTS}. We prove an asymptotically optimal regret bound on the frequentist regret of UTS and we support our claims with simulations showing the significant improvement of our method compared to the state-of-the-art.

\end{abstract}


\section{Introduction}
\label{sec:introduction}

We consider \emph{Stochastic Rank-One Bandits}, a class of bandit problems introduced by \cite{stochasticRank1}.
These models provide a clear framework for the exploration-exploitation problem of adaptively sampling the entries of a rank-one matrix in order to find the largest one.
Consider for instance the problem of finding the best design of a display, say for instance a colored shape to be used as a button on a website.
One may have at hand a set of different shapes, and a set of different colors to be tested.
A display is a combination of those two attributes, and a priori the tester has as many options as there are different pairs of shapes and colors.
Now let us assume  the effect of each factor is independent of the other factor.
Then,  the value of a combination, say for instance the click rate on the constructed button, is the product of the values of each of its attributes.
The better the shape, the higher the rate, and similarly for the color.
This type of independence assumptions is ubiquitous in click models such as the \emph{position-based model} \cite{chuklin2015click,richardson2007predicting}.
 It is also closely related to online learning to rank \cite{zoghi2017online} where sequential duels allow to find the optimal ordering of a list of options. We review the related literature in Section~\ref{sec:related} further below.


We formalize our example above into a Bernoulli rank-one bandit model \citep{rank1Ber}: this model is parameterized by two nonzero vectors $\bm u = (u_1,u_2,..u_K) \in [0,1]^K$ and $\bm v = (v_1,v_2,..v_L) \in [0,1]^L$. There are $K \times L$ arms, indexed by $(i,j) \in [K]\times[L]$, where we use the notation $[p] := \{1,\dots,p\}$ for any positive integer $p$. Each arm $(i,j)$ is associated with a Bernoulli distribution with mean $\mu_{(i,j)} :=u_i v_j$. Observe that the matrix of means $\bm\mu =\bm u\bm v^T$ has rank one, hence the name. We denote $\Theta$ the class of all such instances $(\bm u\times \bm v)$. At each time step $t$ an agent selects an arm $K(t) = (I(t),J(t)) \in [K]\times [L]$ and receives a reward $r(t) \sim \cB(\mu_{(I(t),J(t))})$, independently from previous rewards. To select $K(t)$, the agent may exploit the knowledge of previous observations and possibly some external randomness $U(t)$. Formally, letting $\cF_{t}$ denote the $\sigma$-field generated by $K(1),r(1),K(2),r(2),\dots,K(t),r(t)$, $K(t)$ is measurable with respect to $\sigma(\cF_{t-1},U(t))$.   

The objective of the learner is to adjust their selection strategy to maximize the expected total reward accumulated. The oracle or optimal strategy here is to always play the arm with largest mean. Thus, maximizing rewards is equivalent to designing a strategy $\cA$ with small \emph{regret}, where the $T$-step regret $ R_{\bm\mu}(T,\cA)$ is defined as the difference between the expected cumulative rewards of the oracle and the cumulative rewards of the strategy $\cA$:
\begin{equation}
  \label{eq:regret}
  R_{\bm\mu}(T,\cA) = \sum_{t=1}^T \left[\max_{(i,j)\in [K]\times[L] } \mu_{(i,j)} - \bE_{\bm\mu}[\mu_{(I(t),J(t))}]\right]\;. 
\end{equation}
Letting $i_\star = \text{argmax}_i u_i$ and $j_\star =\text{argmax}_{j} v_j$, we assume that $u_{\istar} > u_i$ for all $i\neq \istar$ and $v_{\star} > v_j$ for all $j \neq \jstar$. This assumption is equivalent to assuming that the rank-one bandit instance has a unique optimal action, which is $(i_\star,j_\star)=\text{argmax}_{(i,j)\in [K]\times[L]} \mu_{(i,j)}$. We let $\Theta_\star$ denote this class of rank-one instance with a unique optimal arm. In this paper, we will furthermore restrict our attention to rank-one models for which either $\bm u \succ 0$ or $\bm v \succ 0$. This assumption is not very restrictive, but it rules out the possibility that $u_i=0$ and $v_j=0$ for a certain arm $(i,j)$ (i.e. neither shape $i$ nor color $j$ attract any user). We found this assumption to be necessary to exhibit a unimodal structure in rank-one bandits. 

An algorithm is called \emph{uniformly efficient} if its regret is sub-polynomial in any instance $(\bm u\times \bm v) \in \Theta$. That is, for all $\alpha >0$, for all $(\bm u\times \bm v) \in \Theta$, $\mathcal{R}(T) = o(T^\alpha)$. In their paper, \cite{stochasticRank1} provide the first uniformly efficient algorithm, $\rankelim$, for stochastic rank-one bandits, and \cite{rank1Ber} propose an adaptation of this algorithm tailored for Bernoulli rewards, $\rankelimKL$.
They also provide a problem-dependent asymptotic lower bound on the regret in the line of \cite{LaiRobbinsLB}. This type of result gives a precise characterization of the regret for a specific instance of the problem that one should expect for any uniformly efficient algorithm. We report their result below.

\begin{proposition}\label{prop:LBClaire} For any algorithm $\cA$ which is uniformly efficient and for any Bernoulli rank-one bandit problem, $(\bm u\times \bm v) \in \Theta_\star$,
$$\liminf\limits_{T\rightarrow \infty} \frac{R_{\bm\mu}(\cA,T)}{\log(T)} \geq \sum_{i \in [K]\setminus i_\star} \frac{\mu_{i_\star,j_\star} - \mu_{i,j_\star}}{\kl(\mu_{i,j_\star},\mu_{i_\star,j_\star})} + \sum_{j \in [L]/j_\star} \frac{\mu_{i_\star,j_\star} - \mu_{i_\star,j}}{\kl(\mu_{i_\star,j},\mu_{i_\star,j_\star})}.$$
where $\kl(x,y) = x \ln(x/y) + (1-x)\ln((1-x)/(1-y))$ is the binary relative entropy.
\end{proposition}

In contrast to this result, the \cite{LaiRobbinsLB} lower bound, which applies to algorithms that are uniformly efficient for \emph{any} reward matrix $\bm\mu$, involves a sum over all matrix entries $(i,j) \in [K]\times[L]$ instead of restricting to arms in the best row and in the best column of the matrix. Thus a good algorithm for the rank-one problem should manage to select all entries $(i,j)$ that are not in this best row and column only $o(\ln(T))$ times. However, neither $\rankelim$ nor $\rankelimKL$ achieve the asymptotic performance of Proposition~\ref{prop:LBClaire}: the regret upper bounds provided by \cite{stochasticRank1,rank1Ber} show much larger constants, and the empirical performance is not much tighter. A natural question one might ask then is: Is that lower bound achievable ?

\paragraph{Contributions}
The main contribution of this paper is to close this existing gap. To do so, we notice and prove that a stochastic rank-one bandit satisfying $\bm u \succ 0$ or $\bm v \succ 0$ is a particular instance of \emph{Unimodal Bandits} \citep{OSUB}.
Interestingly, when derived in the specific rank-one bandits setting, the $\osub$ algorithm proposed in the latter reference achieves the optimal asymptotic regret of Proposition~\ref{prop:LBClaire}.
Unifying those two apparently independent lines of work sheds a new light on stochastic rank-one bandits.

Indeed, follow-up works on unimodal bandits sought ways to construct more efficient algorithms than $\osub$. In particular \cite{UTS} propose $\uts$, a Bayesian strategy based on Thompson Sampling \citep{thompson1933likelihood}. Unfortunately, the theoretical analysis they provide does not allow to conclude an upper bound on the performance of their algorithm. We shall comment on that in Section~\ref{subsec:CommentsProofs}. Thus, a second major contribution of the present work is a new finite-time analysis of the frequentist regret of $\uts$ for Bernoulli stochastic rank-one bandits. Doing so, we provide an optimal regret bound for an efficient and easy-to-implement rank-one bandit algorithm.

Finally, our analysis provides new insights on the calibration of the leader exploration parameter which is present in other algorithms.

\paragraph{Outline} The paper is organised as follows. Section~\ref{sec:unimodal} proves that rank-one bandits are an instance of unimodal bandits, and describes the $\uts$ algorithm. The regret upper bound is proved in Section~\ref{sec:analysis}. In order to perform a fair empirical comparison with existing rank-one bandit algorithms, we give more background on this literature in Section~\ref{sec:related}. Finally, experiments in Section~\ref{sec:experiments} provide empirical evidence of the optimality of $\uts$ and show an improvement of an order of magnitude compared to the state-of-the-art $\rankelimKL$.

\section{Rank-One Bandits, a particular case of Unimodal Bandits}
\label{sec:unimodal}

In this section, we explain why the rank-one bandit model can be seen as a graphical unimodal bandit model as introduced by \cite{yu2011unimodal,OSUB}. For completeness, we recall the relevant definition.

\begin{definition} Given a undirected graph $G=(V,E)$, a vector $\bm\mu =(\mu_k)_{k\in V}$ is \emph{unimodal with respect to} $G$ if
\begin{itemize}
 \item there exists a unique $k_\star \in V$ such that $\mu_{k_\star} = \max_{i} \mu_i$
 \item from any $k \neq k_\star$, we can find an increasing path to the optimal arm. Formally: $\forall k \neq k_\star$, there exists a path $p = (k_1 = k, k_2, ..., k_{m_k} = k_\star)$ of length $m_k$, such that for all $i = 1, ..., m_k-1$, $(k_i,k_{i+1}) \in E$, and $\mu_{k_i} < \mu_{k_{i+1}}$.
\end{itemize}
We denote by $\cU(G)$ the set of vectors $\bm\mu$ that are unimodal with respect to $G$.
\end{definition}

A bandit instance is unimodal with respect to an undirected graph $G=(E,V)$ if its vector of means $\bm\mu = (\mu_k)_{k \in V}$ is unimodal with respect to $G$: $\bm\mu \in \cU(G)$. For a unimodal instance, we define the set of neighbors of an arm $k\in V$ as $\cN(k) = \{\ell : (k,\ell) \in E\}$. Without loss of generality, we can assume that $E$ does not contain self-edges $(k,k)$ (which do not contribute to increasing paths), therefore $k\notin \cN(k)$. The extended neighborhood of $k$ is defined as $\cN^+(k) = \cN(k)\cup \left\{k\right\}$.

In a unimodal bandit problem, the learner knows the graph G (hence the neighborhoods $\cN(k), \cN^+(k)$ for all $k\in V$), but not its parameters $\bm\mu$, which must be learnt adaptively by sampling the vertices of the graph.

\subsection{Rank-One Bandits are Unimodal} \label{unimodal_graph_rank1}

We define the undirected graph $G_1 =(V,E_1)$ as the graph with vertices $V = \{1,\dots,K\} \times \{1,\dots,L\}$ and such that $((i,j),(k,\ell)) \in E_1$ if and only if $(i,j) \neq (k,\ell)$ and ($i=k$ or $j=\ell$). In words, viewing the vertices as a $K \times L$ matrix, two distinct entries are neighbors if they belong to the same line or to the same column. In particular it can be observed that the graph $G_1$ has diameter two, and we shall exhibit below increasing paths of length at most two between any sub-optimal arm $(i,j)$ and the best arm $(i_\star,j_\star)$. 

The main result of this section is Proposition~\ref{prop:RankOneisUnimodal}. It allows us to build on the existing results for unimodal bandits in order to close the remaining theoretical gap in the understanding of rank-one bandits.

\begin{proposition} \label{prop:RankOneisUnimodal}
Let $\bm u = (u_1,u_2,..u_K)$ and $\bm v = (v_1,v_2,..v_L)$ be two nonzero vectors such that $\bm u \succ 0$ or $\bm v \succ 0$. A rank-one bandit instance parameterized by $\bm u, \bm v$ satisfies $\bm\mu \in \cU(G_1)$.
\end{proposition}


\paragraph{Proof} Let $\bm u = (u_1,u_2,..u_K)$ and $\bm v = (v_1,v_2,..v_L)$ be the two vectors parameterizing the rank-one bandit model, and denote  the best arm by $k_\star=(i_\star,j_\star)$. Then for any $(i,j) \in V$ with $(i,j) \neq k_\star$, one can find several increasing paths in $G_1$ from $(i,j)$ to $(i_\star,j_\star)$:
 \begin{itemize}
   \item If $i = i_\star$ or $i = j_\star$, then $(i,j) \rightarrow (i_\star,j_\star)$ is valid as $((i,j),k_\star) \in E_1$ and $\mu_{(i,j)} < \mu_{k_\star}$;
   \item Otherwise, first note that either $v_j\neq 0$ or $u_i \neq 0$. In the first case $(i,j) \rightarrow (i_\star,j) \rightarrow (i_\star,j_\star)$ is a valid increasing path. Indeed, $u_i < u_{i_\star}$ and $0<v_j< v_{j_\star}$ allow us to conclude that $\mu_{(i,j)} = u_i v_j < u_{i_\star} v_j = \mu_{(i_\star,j)} < u_{i_\star}v_{j_\star} = \mu_{(i_\star,j_\star)}$. In the second case, one can similarly show that $(i,j) \rightarrow (i,j_\star) \rightarrow (i_\star,j_\star)$ is a valid increasing path.
 \end{itemize}
\qed
 
  Figure~\ref{fig:Unimodal} below illustrates a possible optimal path in a rank-one bandit with $K=L=4$ and also shows the neighbors of a particular arm in the graph $G_1$.

  \begin{figure}[h]\centering
  $\begin{bmatrix}
      (u_1v_1) & (u_1v_2) & \boldsymbol{(u_1v_3)} & (u_1v_4)  \\
      (u_2v_1) & (u_2v_2) & \boldsymbol{(u_2v_3)} & (u_2v_4)   \\
      \boldsymbol{(u_3v_1)} & \boldsymbol{(u_3v_2)} & \boxed{{(u_3v_3)}} & \boldsymbol{(u_3v_4)}  \\
      (u_4v_1) & (u_4v_2) & \boldsymbol{(u_4v_3)} & (u_4v_4)  \\
  \end{bmatrix}$ \hspace{0.5cm}
  $\begin{bmatrix}
      \boldsymbol{(u_1v_1)} & (u_1v_2) & \boldsymbol{(u_1v_3)} &{(u_1v_4)}  \\
      (u_2v_1) & (u_2v_2) & (u_2v_3) & (u_2v_4)  \\
      (u_3v_1) & (u_3v_2) & \boldsymbol{(u_3v_3)} & (u_3v_4) \\
      (u_4v_1) & (u_4v_2) & (u_4v_3) & (u_4v_4)  \\
  \end{bmatrix}$
  \caption{$\cN((3,3))$ in bold (left). Increasing path from $(3,3)$ to $(i_\star =1, j_\star=1)$ (right).\label{fig:Unimodal}}
  \end{figure}

\subsection{Solving Unimodal Bandits}

In their initial paper, \cite{yu2011unimodal} propose an algorithm based on sequential elimination that does not efficiently exploit the structure of the graph.
\cite{OSUB} take over the unimodal bandit problem and provide a more in-depth analysis of the achievable regret in that setting. In particular,  their Theorem 4.1  states an asymptotic regret lower bound that we state below for Bernoulli rewards.

\begin{proposition} \label{prop:LBRichard}
  Let $G=(V,E)$ define a Bernoulli unimodal bandit problem, with $\cN_G(k) = \{\ell : (k,\ell) \in E\}$ denoting the set of neighbors of arm $k \in V$. Let $\cA$ be a uniformly efficient algorithm  for every Bernoulli bandit instance with means in $\cU(G)$.
  Then
\[\forall \bm\mu \in \cU(G), \ \ \liminf_{T\rightarrow \infty}  \frac{R_{\bm\mu}(\cA,T)}{\ln(T)} \geq \sum_{k \in \cN_G(k_\star)} \frac{\mu_{k_\star} - \mu_k}{\kl\left(\mu_k,\mu_{k_\star}\right)}.\]
\end{proposition}

In the particular case $G=G_1$, $\cN_{G_1}((\istar,\jstar)) = \{(i,j) : i=i_\star \text{ or }  j= j_\star \}\backslash \{(\istar,\jstar)\}$ and we recover Proposition~\ref{prop:LBClaire}. An asymptotically optimal algorithm for unimodal bandits therefore particularizes into an asymptotically optimal algorithm for rank-one bandits.

\subsection{Candidate algorithms and their analysis}

\label{subsec:CommentsProofs}

There exists only a few optimal algorithms for unimodal bandits.
\cite{OSUB} propose $\osub$, a computationally efficient algorithm that is proved to have the best achievable regret. \cite{UTS} propose a Bayesian alternative, however for reasons detailed below we believe their regret analysis does not hold as is.
Another valid algorithm would be $\ossb$ \citep{combes2017minimal}, a generic method for structured bandits, however its implementation for rank-one bandits is not obvious (the matrix of empirical mean $\hat{\bm\mu}(t)$ would need to have rank one), and its generality often makes it less empirically efficient when compared to algorithms exploiting a particular structure, like here the rank-one structure.

\paragraph{Notation} We now present the existing algorithms for unimodal bandits with respect to some undirected graph $G=(V,E)$. For $k \in V$, we let $N_k(t) = \sum_{s=1}^t \ind_{(K(s)= k)}$ be the number of selections of arm $k$ up to round $t$ and $\hat{\mu}_k(t) = \frac{1}{N_k(t)} \sum_{s=1}^t X_s \ind_{(K(s) = k)}$ be the empirical means of the rewards from that arm. We also define the (empirical) leader $L(t)= \text{argmax}_{k \in V} \ \hat{\mu}_k(t)$ and keep track of how many times each arm has been the leader in the past by defining $\ell_{k}(t) = \sum_{s=1}^t \ind_{\left(L(s) = k\right)}$.

\paragraph{Optimal Sampling for Unimodal Bandits ($\bm\osub$)}

$\osub$ \citep{OSUB}  is the adaptation of the $\kl$-UCB algorithm of \cite{cappe2013kullback}, an asymptotically optimal algorithm for (unstructured) Bernoulli bandits.
The vanilla $\kl$-UCB algorithm uses as upper confidence bounds the indices
\[\UCB_k(t) = \max \left\{ q : N_k(t) \kl\left(\hat{\mu}_k(t),q\right) \leq f(t) \right \},\]
and selects at each round the arm with largest index.

The idea of $\osub$ is to restrict $\kl$-UCB to the neighborhood of the leader while adding a \emph{leader exploration mechanism} to ensure that the leader gets ``checked'' enough and can eventually be trusted. Letting
\begin{equation}\widetilde{\UCB}_k(t) = \max \left\{ q : N_k(t) \kl\left(\hat{\mu}_k(t),q\right) \leq f(\ell_{L(t)} (t)) \right \},\label{def:IndexOSUB}\end{equation}
$\osub$ selects at time $t+1$
\begin{equation}A_{t+1} = \left\{\begin{array}{cl}
                   L(t) & \text{if } \ell_{L(t)} (t) \equiv 1 [\gamma ], \\
                    \argmax{k} \ \widetilde{\mathrm{u}}_k(t) & \text{else.}
                   \end{array}
\right.\label{def:LeaderExplo}\end{equation}
The parameter $\gamma$ quantifies how often the leader should be checked. $\osub$ is proved to be asymptotically optimal when $\gamma$ is equal to the maximal degree in $G +1$, which yields $\gamma = K + L -1$ for rank-one bandits. Compared to $\kl$-UCB, the alternative exploration rate $f(\ell_{L(t)} (t))$ that appears in the index \eqref{def:IndexOSUB} makes the analysis of $\osub$ quite intricate.


\paragraph{Unimodal Thompson Sampling ($\bm\uts$)}

For classical bandits, Thompson Sampling (TS) is known to be a good alternative to $\kl$-UCB as it shares its optimality property for Bernoulli distributions \citep{ALT12,thompson_samplingAG} without the need to tune any confidence interval 
and often with better performance. \cite{UTS} therefore naturally proposed Unimodal Thompson Sampling (UTS). The algorithm, described in detail in Section~\ref{sec:UTS}, consists in running Thompson Sampling instead of $\kl$-UCB in the neighborhood of the leader, while keeping a leader exploration mechanism similar to the one in \eqref{def:LeaderExplo}. The exploration parameter $\gamma$ should also be set to $K+L-1$ in the rank-one case in order to prove the asymptotic optimality of UTS.

The analysis proposed by \cite{UTS} (detailed in Appendix A of the extended version \cite{UTSarXiv}) hinges on adapting some elements of the Thompson Sampling proof of \cite{ALT12} and is not completely satisfying. Our main objection is the upper bound that is proposed on the number of times a sub-optimal arm $k$ is the leader (term $\cR_2$  of the second equation on page 8). To deal with this term, a quite imprecise reduction argument is given (definition of $\hat L_{k,t}$) showing that one essentially needs to control the quantity $\sum_{t=1}^T \bP\left(\hat{\mu}_k(t) \geq \hat{\mu}_{k_2}(t)\right)$ for Thompson Sampling playing in $\cN(k)$ and $k_2$ being the element with largest mean in this neighborhood. However, we do not believe this quantity can be easily controlled for Thompson Sampling, as we have to handle a random number of observations (that may be small) from both $k$ and $k_\star$. Besides, the upper bound on $\cR_2$ proposed by \cite{UTS} holds for the choice $\gamma = K+L-1$ in the rank-one case, which we show is unnecessary.

Due to the lack of accuracy of the existing proof, we believe that a new, precise analysis of Unimodal Thompson Sampling is needed to corroborate its good empirical performance for rank-one bandits, which we provide in the next section. Our analysis borrows elements from both the TS analysis of \cite{thompson_samplingAG} and that of \cite{ALT12}. It also reveals that unlike what was previously believed, the leader exploration parameter can be set to an arbitrary value $\gamma \geq 2$.

\section{Analysis of Unimodal Thompson Sampling}
\label{sec:analysis}

In this section, we present the \emph{Unimodal Thompson Sampling} algorithm ($\uts$) for Bernoulli rank-one bandits, and we state our main theorem proving a problem-dependent regret upper bound for this algorithm, which extends to the graphical unimodal case.

    \subsection{UTS for Rank-One Bandits}

$\uts$ is a very simple computationally efficient, anytime algorithm. Its pseudo-code for Bernoulli rank-one bandits is given in Algorithm~\ref{alg:UTS}. It relies on one integer parameter $\gamma \geq 2$ controlling the fraction of rounds spent exploring the leader.
After an initialization phase where each entry is pulled once, at each round $t>K\times L$, the algorithm computes the leader $L(t)$, that is the empirical best entry in the matrix. If the number of times $L(t)$ has been leader is multiple of $\gamma$, $\uts$ selects the empirical leader. The rest of the time, it draws a posterior sample for every entry in the same row and column as the leader, and selects the entry associated to the largest posterior sample. This can be viewed as performing Thompson Sampling in $\cN_{G_1}^+(L(t))$, the augmented neighborhood of the leader in the graph $G_1$ defined in Section~\ref{sec:unimodal}.

    \label{sec:UTS}

    \begin{algorithm}
      \caption{$\uts$ for Bernoulli rank-one bandits \label{alg:UTS}}
      \begin{algorithmic}
        \State \textbf{Input:} $\gamma \in \mathbb{N}, \gamma \geq 2$.
        \For{$(i,j) \in [K]\times[L]$}
        \State $N_{(i,j)} = 1$. $L_{(i,j)} = 0$.
        \State Draw arm $(i,j)$, receive reward $R$ and let $S_{(i,j)} = R$.
        \EndFor
        \For{$t = KL +1,\dots,T$}
          \State Compute the entry-wise empirical leader $L(t) = \argmax{(i,j) \in [K]\times [L]} \hat{\mu}_{i,j}(t)$;

          \vspace{-0.3cm}

          \State Update the leader count $L_{L(t)} \gets L_{L(t)} +1$

          \If{$L_{L(t)} \equiv 0\, [\gamma]$},
            \State $(I(t),J(t))=L(t)$
          \Else
            \For{$k \in \{(I(t),j):\, j\in [L]\} \bigcup \{(i,J(t)):\, i\in [K]\}$}
            \State $\theta_k \sim \text{Beta}\left(S_k +1, N_k - S_k + 1\right)$
            \EndFor
            \State $(I(t),J(t)) = \argmax{k} \ \theta_k$.
          \EndIf
          \State Receive reward $R_t \sim \cB(\mu_{(I_t,J_t)})$
          \State $N_{(I(t),J(t))} \gets N_{(I(t),J(t))} +1 $. $S_{(I(t),J(t))} \gets S_{(I(t),J(t))} +R_t $
        \EndFor
      \end{algorithmic}
    \end{algorithm}

    For completeness, we recall that given a prior distribution Thompson Sampling maintains a posterior distribution for each hidden Bernoulli parameter of the problem, that is, for each entry of the matrix. To do so, it uses a convenient uniform ($\text{Beta}(1,1)$)  prior, for which the posterior distribution is a $\text{Beta}$ distribution. We refer the interested reader to the recent survey \cite{russo2018tutorial} for more details on the topic.

\subsection{Regret upper bound and asymptotic optimality}

$\uts$ can be easily extended to any graphical unimodal bandit problem with respect to a graph $G$, by performing Thompson Sampling on $\cN_G^+$ instead of $\cN_{G_1}^+$. For this more general algorithm, we state the following theorem, which is our main technical contribution.

\begin{theorem}\label{theo_ub_uts}Let $\bm\mu$ be a graphical unimodal bandit instance with respect to a graph $G$. For all $\gamma \geq 2$, UTS with parameter $\gamma$ satisfies, for every $\epsilon > 0$,
         $$\mathcal{R}_{\bm\mu}(T,\uts(\gamma)) \leq (1+\epsilon)\sum_{k \in \cN(k_\star)}  \frac{(\mu_\star - \mu_k)}{\kl(\mu_k,\mu_\star)}\ln(T) + C(\bm\mu,\gamma,\epsilon),$$
where $C(\bm\mu,\gamma,\epsilon)$ is some constant depending on the environment $\bm\mu$, on $\epsilon$ and on $\gamma$.
\end{theorem}

A consequence of this finite-time bound is that, for every parameter $\gamma \geq 2$,
\[\limsup_{T\rightarrow\infty}\frac{\mathcal{R}_{\bm\mu}(T,\uts(\gamma))}{\ln(T)} \leq \sum_{k \in \cN(k_\star)}  \frac{(\mu_\star - \mu_k)}{\kl(\mu_k,\mu_\star)},\]
therefore $\uts(\gamma)$ is asymptotically optimal for any graphical unimodal bandit problem. Particularizing this result to rank-one bandits, one obtains that Algorithm~\ref{alg:UTS} has a regret which is asymptotically matching the lower bound in Proposition~\ref{prop:LBClaire}.

Unlike previous work, in which logarithmic regret is proved only for the choice $\gamma = K+L -1$ in the rank-one case\footnote{For general unimodal bandits, \osub{} sets $\gamma$ to be the maximal degree of an arm, whereas \uts{} adaptively sets $\gamma$ to be the degree of the current leader. Both parameterization coincide for rank-one bandits.}, we emphasize that this result holds for any choice of the leader exploration parameter. We conjecture that $\uts$ without any leader exploration scheme is also asymptotically optimal. However, our experiments of Section~\ref{sec:experiments} reveal that this particular kind of ``forced exploration'' is not hurting for rank-one bandits, and that the choice $\gamma = 2$ actually leads to the best empirical performance.

    \subsection{Proof of Theorem~\ref{theo_ub_uts}} \label{sec:proof_uts}
 We consider a general $K$-armed graphical unimodal bandit problem with respect to some graph $G$ and let $K(t)$ denote the arm selected at round $t$. We recall some important notations defined in Section~\ref{subsec:CommentsProofs}: the number of arms selections $N_k(t)$, the empirical means $\hat{\mu}_k(t)$, the
leader as $L(t) = \text{argmax}_{k} \ \hat{\mu}_k(t)$, and the
    number of times arm $k$ has been the leader up to time $t$:
    $\ell_k(t) = \sum_{s=1}^t \ind\left(L(s) = k\right)$.
    Observe that the leader exploration scheme ensures that \begin{equation}\forall k \in \{1,\dots,K\}, \forall t\in \N, N_k(t) \geq \lfloor \ell_k(t) / \gamma\rfloor.\label{LeaderExplo}\end{equation}

    Introducing the gap $\Delta_k = \mu_\star - \mu_k$, recall that the regret rewrites $\sum_{k \neq k_\star} \Delta_k \bE_{\bm\mu} [N_k(T)]$. Just like in the analysis of \cite{OSUB,UTS}, we start by distinguishing the times when the leader is the optimal arm, and the times when the leader is a sub-optimal arm:
\begin{align*}
&\cR_{\bm\mu}(T,\text{UTS}(\gamma)) = \sum_{k \neq k_\star} \Delta_k \mathbb{E} \left[ \sum_{t=1}^T \mathbbm{1} (K(t)=k) \right] \\
&= \underbrace{\sum_{k \neq k_\star} \Delta_k \mathbb{E} \left[ \sum_{t=1}^T \mathbbm{1} (K(t)=k, L(t) = k_\star) \right]}_{\mathcal{R}_1(T)} + \underbrace{\sum_{k \neq k_\star} \Delta_k \mathbb{E} \left[ \sum_{t=1}^T \mathbbm{1} (K(t)=k, L(t) \neq k_\star) \right]}_{\mathcal{R}_2(T)}\;.
\end{align*}

To upper bound $\mathcal{R}_1 (T)$, it can be noted that when $k_\star$ is the leader, the selected arm $k$ is necessarily in the neighborhood of $k_\star$, hence the sum can be restricted to the neighborhood of $k_\star$. Therefore, we expect to upper bound $\mathcal{R}_1 (T)$ by the same quantity which upper bounds the regret of Thompson Sampling restricted to $\cN^+(k_\star)$. Such an argument is used for KL-UCB and Thompson Sampling by \cite{OSUB} and \cite{UTS} respectively, without much justification. However, a proper justification does need some care, as between two times the leader is $k_\star$, UTS may update the posterior of some arms in $\cN^+(k_\star)$ for they belong to the neighborhoods of other potential leaders.

In this work, we carefully adapt the analysis \cite{thompson_samplingAG}
to get the following upper bound. The proof can be found in Appendix~\ref{appendix:R_1_bound}.

\begin{lemma} \label{R_1_bound} For all $\epsilon >0$ and all $T \geq 1$,
$$\mathcal{R}_1(T) \leq (1+\epsilon)\sum_{k \in N(k_*)}  \frac{\Delta_k}{\kl(\mu_k,\mu_\star)} \ln(T) +\Tilde{C}(\bm\mu,\epsilon),$$
for some quantity $\Tilde{C}(\bm\mu,\epsilon)$ which depends on the means $\bm\mu$ and on $\epsilon$ but not on $T$.
\end{lemma}

We now upper bound $\cR_2(T)$, which can be related to the probability of choosing any given suboptimal arm $k$ as the leader:
\begin{eqnarray*}
\cR_2(T) &\leq & \sum_{\ell \neq k_\star}\sum_{k \neq k_\star} \Delta_k \mathbb{E} \left[ \sum_{t=1}^T \mathbbm{1} (K(t)=k, L(t) = \ell) \right] \\ & \leq  &\sum_{\ell \neq k_\star} \sum_{t=1}^T\bE\left[\ind(L(t) = \ell) \sum_{k\neq k_\star} \ind(K(t) = k)\right]  =  \sum_{k \neq k_\star} \sum_{t=1}^T\bP\left(L(t) = k\right).
\end{eqnarray*}
For each $k\neq k_\star$, we define the set of best neighbors of $k$, $\cB_{\cN(k)} = \text{argmax}_{\ell\in \cN(k)} \mu_\ell$. Due to the unimodal structure, we know this set is nonempty because there exists at least one arm $\ell \in \cN(k)$ such that $\mu_\ell > \mu_k$ (such an arm belongs to the path from $k$ to $k_\star$). All arms belonging to $\cB_{\cN(k)}$ have same mean, that we note $\mu_{k_2} = \max_{\ell\in \cN(k)} \mu_\ell$. We also introduce $\tilde{B} = \max_{k \in [K] \setminus\{k_\star\}} |\cB_{\cN(k)}|$, the maximal number of best arms in the neighborhood of all sub-optimal arms, which is bounded by the maximum degree of the graph. With this notation, one can write, for any $b \in (0,1)$,
\begin{eqnarray*}
\sum_{t=1}^T\bP\left(L(t) = k\right) &= & \underbrace{\sum_{t=1}^T\bP\left(L(t) = k, \exists k_2 \in \cB_{\cN(k)}, N_{k_2}(t) > (\ell_k(t))^b\right)}_{\cT^k_1(T)} \\
&& \hspace{0.2cm} + \underbrace{\sum_{t=1}^T\bP\left(L(t) = k, \forall k_2 \in \cB_{\cN(k)}, N_{k_2}(t) \leq (\ell_k(t))^b\right)}_{\cT^k_2(T)}
\end{eqnarray*}

The first term can be easily upper bounded by using the fact that if both arm $k$ and one of its best neighbors $k_2 \in \cB_{\cN(k)}$ are selected enough, it is unlikely that $\hat{\mu}_k(t) \geq \hat{\mu}_{k_2}(t)$. 

On the event $\{L(t) = k\}$, the empirical mean of the $k$-th arm is necessarily greater than that of the other arms (especially those in $\cB_{\cN(k)}$) . Therefore, letting $\delta_k = \frac{\mu_{k_2} - \mu_k}{2}$,

\begin{eqnarray}
 \cT_1^k(T) & =&  \sum_{t=1}^T \bP\left(L(t)=k, \exists k_2 \in \cB_{\cN(k)}, \hat{\mu}_k(t) \geq \hat{\mu}_{k_2}(t), N_{k_2}(t) > (\ell_k(t))^b\right)\nonumber \\
 & \leq &  \sum_{t=1}^T \bP\left(L(t)=k, \hat{\mu}_k(t) > \mu_k + \delta_k, N_{k}(t) > \lfloor\ell_k(t)/\gamma\rfloor\right) \label{FirstPart}\\
 && \hspace{0.2cm} + \sum_{t=1}^T \bP\left(L(t)=k, \exists k_2 \in \cB_{\cN(k)}, \hat{\mu}_{k_2}(t) \leq \mu_{k_2} - \delta_k , N_{k_2}(t) > (\ell_k(t))^b\right), \label{SecondPart}
\end{eqnarray}
where in \eqref{FirstPart}, we have used the leader exploration mechanism \eqref{LeaderExplo}. \eqref{FirstPart} and \eqref{SecondPart} can be upper bounded in the same way, by introducing the sequence of stopping times $(\tau_{i}^k)_i$, where $\tau_{i}^k$ is the instant at which arm $k$ is the leader for the $i$-th time (one can have $\tau_{i}^k > T$ or $\tau_{i}^k = +\infty$ if arm $k$ would be the leader only a finite number of time when $\uts$ is run forever).
\begin{align*}
\eqref{SecondPart} & \leq \sum_{k_2 \in \cB_{\cN(k)}} \sum_{i=1}^T\sum_{t=1}^T \mathbb{E}[\ind(L(t)=k,\ell_k(t) = i, \hat{\mu}_{k_2}(t) \leq \mu_{k_2} - \delta_k, N_{k_2}(t)>i^b)] \\
&= \tilde{B} \sum_{i=1}^T \bP\left(\hat{\mu}_{k_2}(\tau_{i}^k) \leq \mu_{k_2} - \delta_k , N_{k_2}(\tau_i^k) > i^b,  \tau_i^k \leq T\right) \\
&\leq \tilde{B} \sum_{i=1}^T \sum_{u=i^b}^T \bP\left(\hat{\mu}_{k_2,u} \leq \mu_{k_2} -  \delta_k, N_{k_2}(\tau_i^k) = u \right) \\
&\leq \tilde{B} \sum_{i=1}^\infty \sum_{u=i^b}^\infty \exp(-2\delta_k^2u) \leq \tilde{B}\sum_{i=1}^\infty \frac{\exp(-2\delta_k^2i^b)}{1-\exp(-2\delta_k^2)}.
\end{align*}
The notation $\hat{\mu}_{k_2,u}$ used above denotes the empirical mean of the first $u$ observations from arm $k_2$, which are i.i.d. with mean $\mu_{k_2}$. Thus, Hoeffding's inequality can be applied to obtain the last but one inequality.

To upper bound \eqref{FirstPart} we use the same approach (with $i^b$ replaced by $\lfloor i/\gamma\rfloor$), which yields
\[\cT_1^k(T) \leq \sum_{i=1}^\infty \frac{\exp(-2\delta_k^2i^b)}{1-\exp(-2\delta_k^2)} + \sum_{i=1}^\infty \frac{\exp(-{2\delta_k^2}\lfloor i/\gamma\rfloor)}{1-\exp(-2\delta_k^2)} := C_k(\bm\mu,\gamma,b) < \infty.\]

To finish the proof, we upper bound $\cT_2^k(T)$ for some well chosen value of $b \in (0,1)$. The upper bound given in Lemma~\ref{lem:TermeRelou} is a careful adaptation (and generalization) of the proof of Proposition 1 in \cite{ALT12}, which says that for vanilla Thompson Sampling restricted to $\cN^+(k_\star)$, the (unique) optimal arm $k_2$ cannot be drawn too few times. Observe that Lemma~\ref{lem:TermeRelou} permits to handle possible multiple optimal arms.
Again, we emphasize that in UTS, there is an extra difficulty due to the fact that arms in $\cN^+(k_\star)$ are not only selected when $k$ is the leader. The proof of Lemma~\ref{lem:TermeRelou}, given in Appendix~\ref{appendix:TermeRelou} overcomes this difficulty.

\begin{lemma}\label{lem:TermeRelou} When $\gamma \geq 2$, there exists $b\in (0,1)$ and a constant $D_k(\bm\mu,b,\gamma)$ such that
\[\sum_{t=1}^T \bP\left(L(t) = k, \forall k_2 \in \cB_{\cN(k)}, N_{k_2}(t) \leq (\ell_k(t))^b\right) \leq D_k(\bm\mu,b,\gamma).\]
\end{lemma}

Putting things together, one obtains, for all $\epsilon>0$, with $b$ chosen as in Lemma~\ref{lem:TermeRelou},
\[\cR_{\bm\mu}(\cA,T) \leq (1+\epsilon)\sum_{k \in N(k_*)}  \frac{\Delta_k}{\kl(\mu_k,\mu_\star)} \ln(T) + \Tilde{C}(\bm\mu,\epsilon) + \sum_{k\neq k_\star}\left[C_k(\bm\mu,\gamma,b) + D_k(\bm\mu,b,\gamma)\right],\]
which yields the claimed upper bound. 

\section{Related Work on Rank-One Bandits}
\label{sec:related}

Multi-armed bandits are a rich class of statistical models for sequential decision making (see \cite{lattimore2018bandit, bubeck2012regret} for two surveys).
They offer a clear framework as well as computationally efficient algorithms for many practical problems such as online advertising \cite{zoghi2017online}, a context in which the empirical efficiency of Thompson Sampling \citep{thompson1933likelihood} has often been noticed  \citep{Scott10,chapelle2011empirical}.
The wide success of Bayesian methods in bandit or reinforcement learning problems can no longer be ignored \cite{russo2018tutorial,osband2017posterior}.

As already mentioned, stochastic rank-one bandits were introduced by \cite{stochasticRank1,rank1Ber} which are indeed among the closest works related to ours.
The original algorithm proposed therein, $\rankelim$, relies on a complex sequential elimination scheme.
It operates in stages that progressively quadruple in length. At the end of each stage, the significantly worst rows and columns are eliminated; this is done using carefully tuned confidence intervals.
The exploration is simple but costly: every remaining row is played with a randomly chosen remaining column, and conversely for the columns.
At the end of the stage, the value of each row is computed by averaging over all columns, such that the estimate of the row parameter is scaled by some measurable constant that is \emph{the same} for all rows.
Then,  $\ucb$ or $\klucb$ confidence intervals are used to perform the elimination by respectively $\rankelim$ or $\rankelimKL$.
The advantage of this method is that the worst rows and columns disappear very early from the game. However, eliminating them requires that their confidence intervals no longer intersect, which is quite costly.
Moreover, the averaging performed to compute individual estimates for each parameter may be arbitrarily bad: if all columns but one have a parameter close to zero, the scaling constant on the row estimates is close to zero and the rows become hard to distinguish.
All those issues are mentioned in the according papers.
Nonetheless, the advantage of a rank-one algorithm, as opposed to playing a vanilla bandit algorithm, on a large (typically $64\times64$) matrix remains perfectly significant, which has motivated various further work on the topic.

In particular, \cite{kveton2017stochastic} generalizes this elimination scheme to low-rank matrices, where the objective is to discover the $d\times d$ best set of entries. \cite{jun2019bilinear} modify a bit the problem and formulate it as \emph{Bilinear bandits}, where the two chosen vector arms $x_t$ and $y_t$ have an expected payoff of $x_t^\top M y_t$, where $M$ is a low-rank matrix. \cite{kotlowski2019bandit} study an adversarial version of this problem, the \emph{Bandits Online PCA}: the learner sequentially chooses vectors $x_t$ and observes a loss $x_tx_t^\top L_t$, where the loss is arbitrarily and possibly adversarially chosen by the environment. \cite{zimmert2018factored} considers a more general problem where matrices are replaced by rank-one tensors in dimension $d\geq 2$. The main message of the paper is to propose a unified view of \emph{Factored Bandits} encompassing both rank-one bandits and dueling bandits \cite{yue2009interactively}.

\section{Numerical Experiments}
\label{sec:experiments}

To assess the empirical efficiency of $\uts$ against other competitors, we follow the same experimental protocol as \cite{rank1Ber} and run the algorithm on simulated matrices of arms of increasing sizes.
We set $K=L$ for different values of $K$. The parameters are defined symmetrically: $\bm u=\bm v =(0.75,0.25, \ldots, 0.25)$ such that the best
entry of the matrix is always $(i^*,j^*)=(1,1)$. In our experiments, the cumulative regret up to an horizon $T=300000$ is estimated based on $100$ independent runs. The shaded areas on our plots show the 10\% percentiles.

\paragraph{Study of hyperparameter $\bm\gamma$} \label{UTS_study_gamma}
According to the original paper, the exploration parameter of $\uts$ should be set to $\gamma = K+L-1$ for rank-one bandits. However, in the proof we derived in Section \ref{sec:analysis}, there is no need to fix $\gamma$ to this value. To confirm this statement and study the influence of $\gamma$, we ran  $\uts$ on a the $K=4$ toy problem described above, with different values of $\gamma \in \{2,\, 5,\, 10,\, 20\}$. We also run the heuristic version of $\uts$ that would use no leader exploration scheme (corresponding to $\gamma = +\infty$).

On Figure~\ref{fig:study_gamma_4_4}, we show the cumulative regret in log-scale. We notice that all curves align with the optimal logarithmic rate, with a lower offset for lower values of $\gamma$. Empirically, the performance seems the best for $\gamma=2$.

 \begin{figure}[h]
     \begin{center}
     \includegraphics[width=0.6\textwidth]{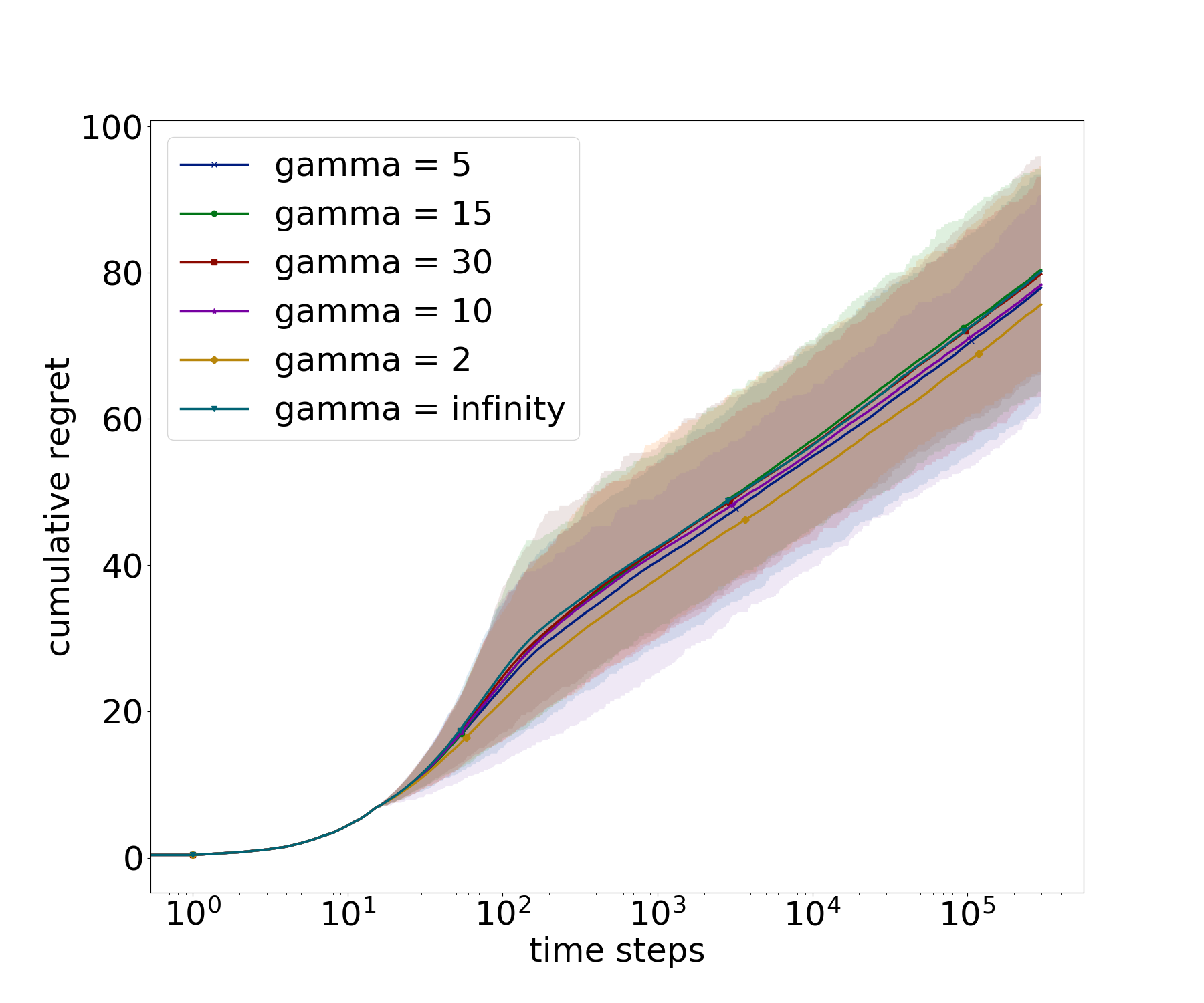}
     \caption{Cumulative regret of $\uts$ for $\gamma$ varying in $\{2,5,10,20,+\infty\}$ for $K=4$.\label{fig:study_gamma_4_4}}
   \end{center}

 \end{figure}

 \begin{figure}[p]
  \includegraphics[width=0.45\textwidth]{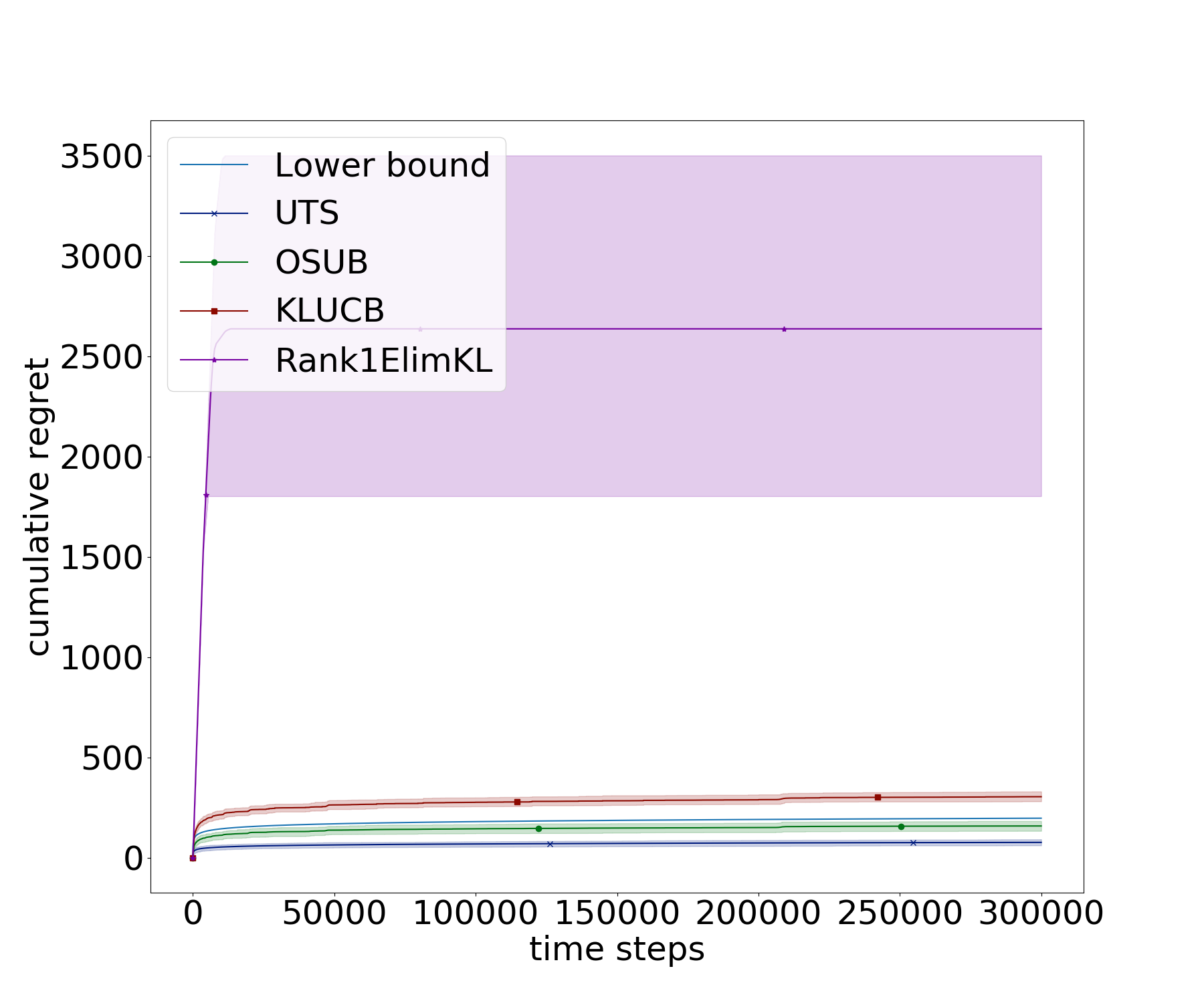}
  \includegraphics[width=0.45\textwidth]{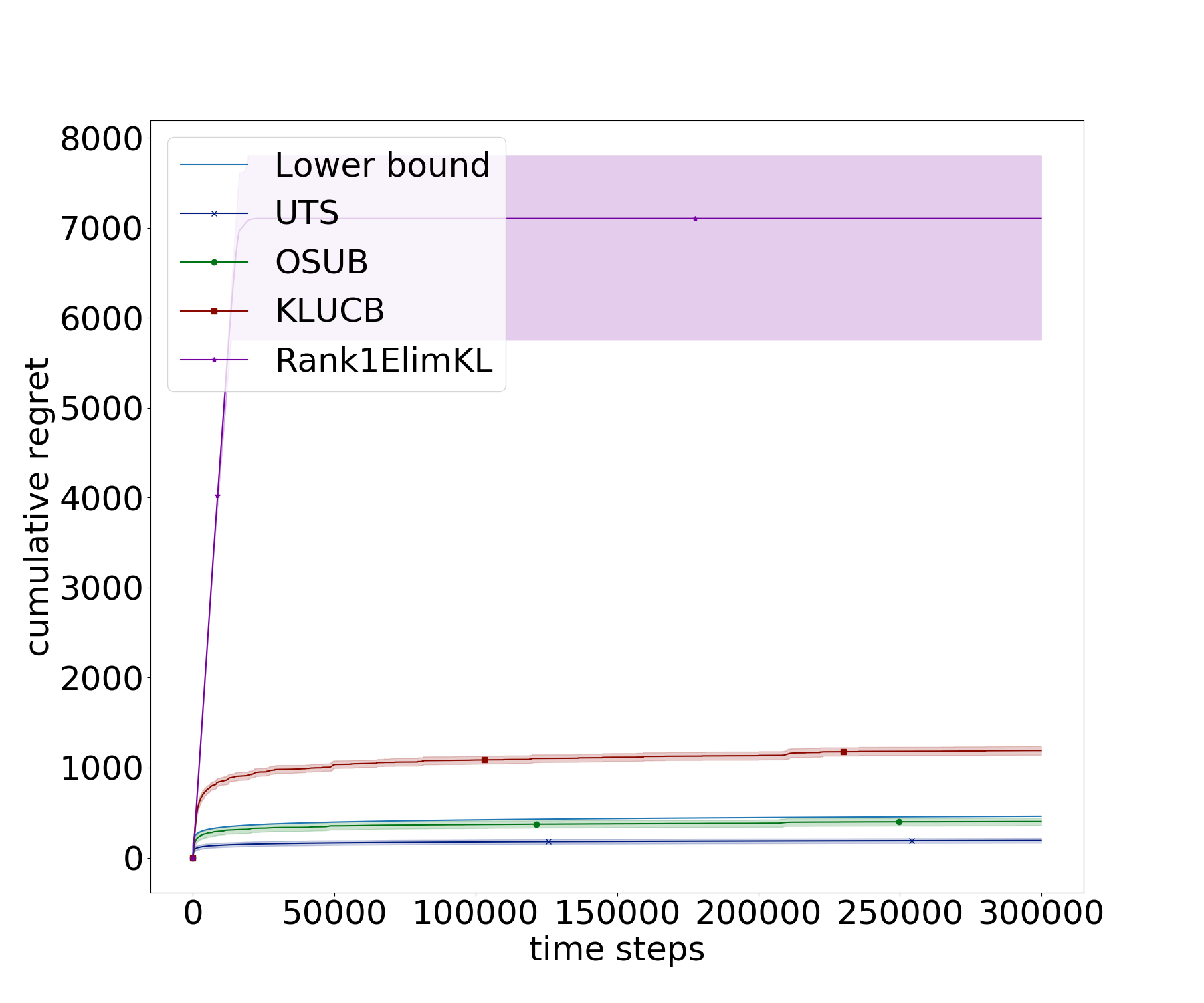}
  \centering
  \includegraphics[width=0.45\textwidth]{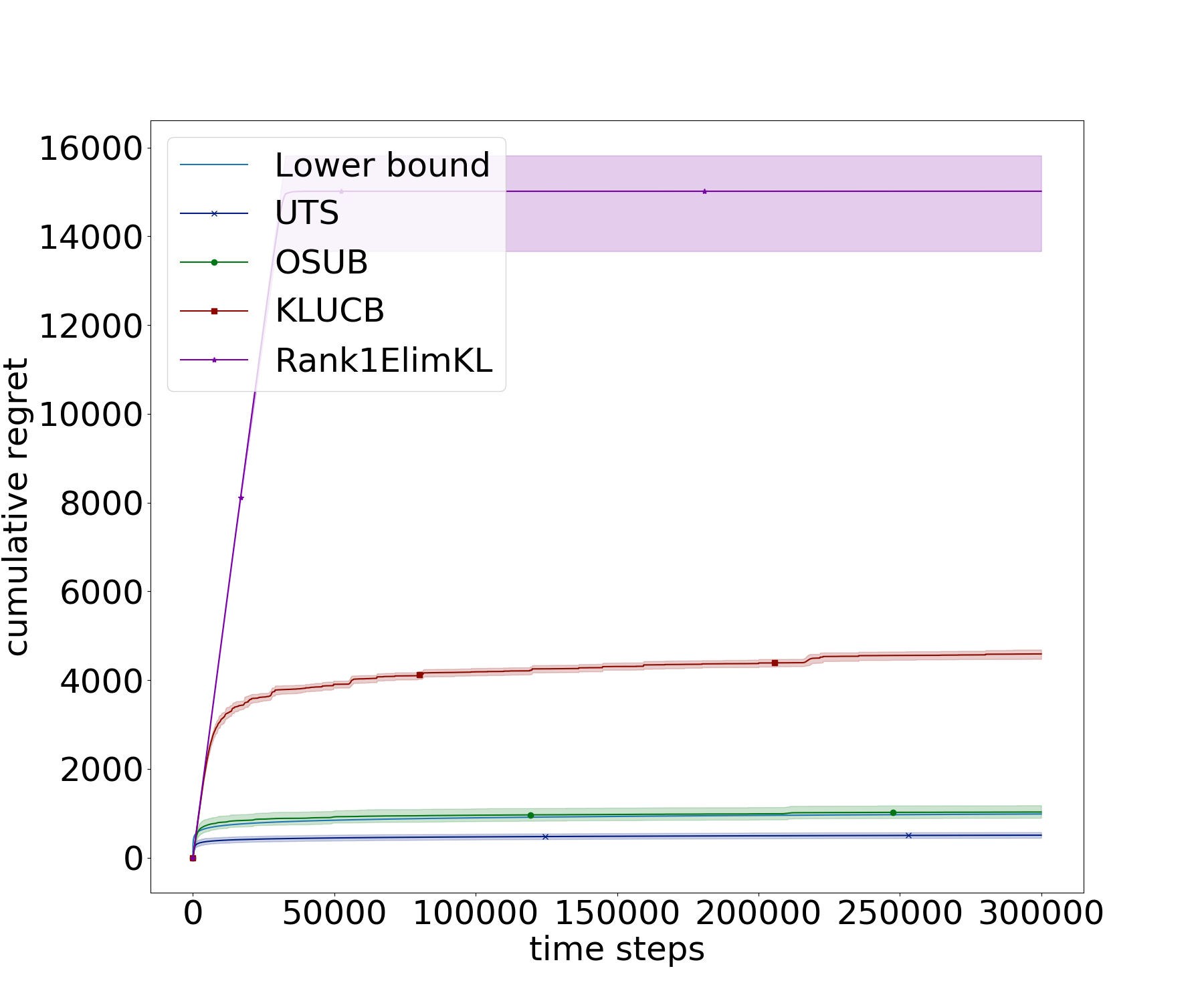}
  \caption{Cumulative regret of $\rankelimKL$, $\osub$, $\uts$ and $\klucb$, on $K\times K$ rank-one matrices with $K=4$ (top left), $K=8$ (top right) and $K=16$ (bottom) \label{fig:regrets}}
\end{figure}

\paragraph{Cumulative regret and optimality of $\bm\uts$. } We now compare the regret of $\uts$ run with $\gamma=2$ to that of other algorithms on the above mentioned family of instances for different values of $K$ in $\{4,8,16\}$. Note that in \cite{rank1Ber}, the simulations are run on larger matrices, for $K=32,64,128$.
In those settings, $\rankelimKL$ only outperforms $\klucb$ for $K=128$ but it is  better than $\rankelim$ and one can easily see that it scales better with the problem size than UCB1.
However, given the much better performance of $\uts$ and $\osub$, we were able to show the same trends with much smaller problem sizes.

In Figure~\ref{fig:regrets} we compare the cumulative regret of $\rankelimKL$ with $\osub$, $\uts$ (with $\gamma=2$) and $\klucb$.
One first obvious observation is that $\rankelimKL$ has a regret an order of magnitude larger than all other policies, including $\klucb$ on this size of problems.
 We also notice that the final regret, at $T=300K$, roughly doubles for all rank-one policies while it quadruples for $\klucb$, as expected.
To illustrate the asymptotic optimality of $\osub$ and $\uts$ compared to $\klucb$, we show on Figure~\ref{fig:logregret} the results of the $K=4$ simulations in log-scale, and we plot the lower bound of Proposition~\ref{prop:LBClaire}.
 We observe that both optimal policies asymptotically align with the lower bound, while $\klucb$ adopts a faster growth rate, that indeed corresponds to the constant of Lai \& Robbins.

\begin{figure}[h]
  \begin{center}
    \includegraphics[width=0.5\textwidth]{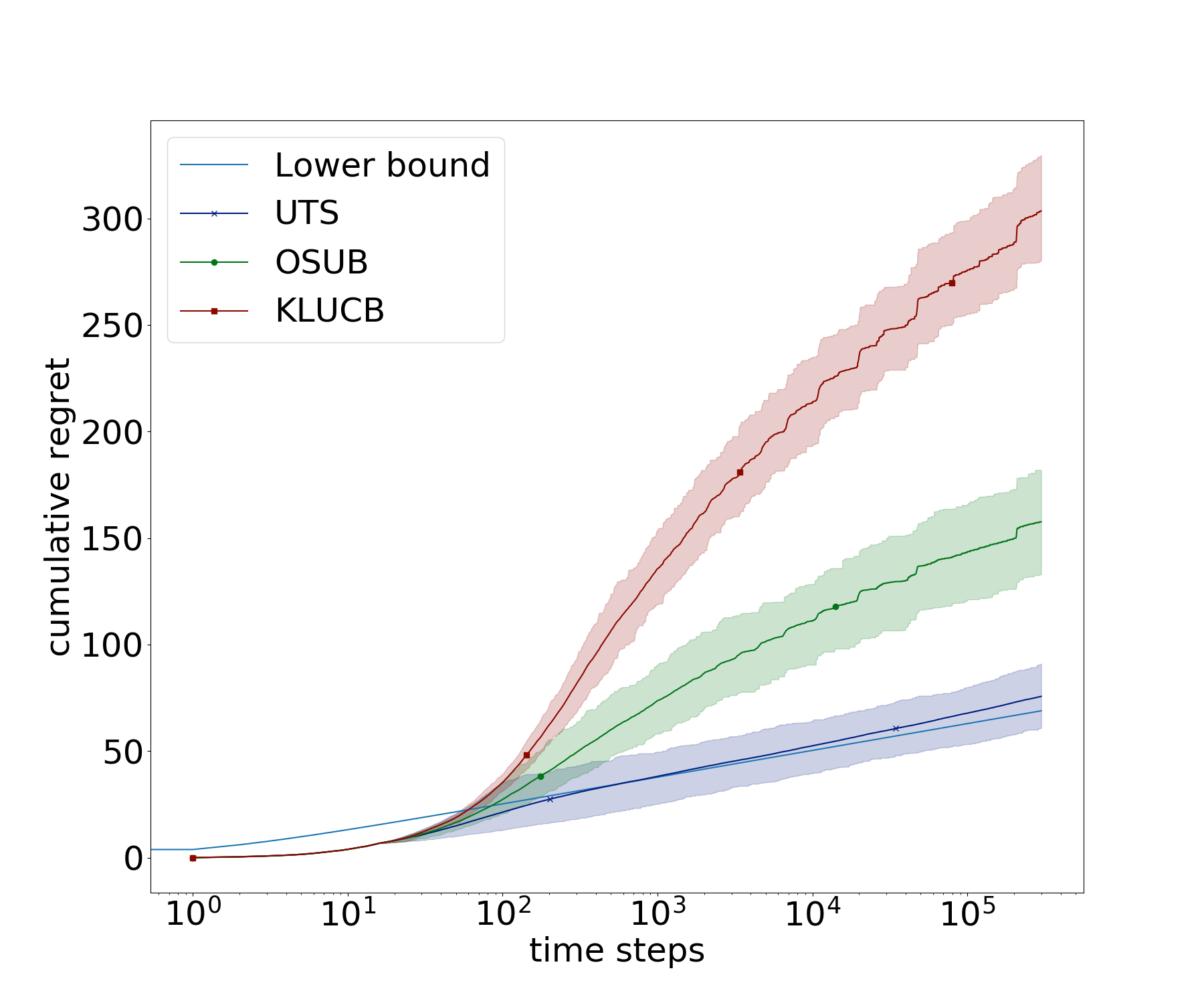}
  \end{center}
  \caption{Regret for $K=4$ in log-scale: the lower bound (in blue) shows the optimal asymptotic logarithmic growth of the regret. $\uts$ and $\osub$ align with it, while $\klucb$ has a larger slope. }\label{fig:logregret}
\end{figure}

\section{Conclusion}

This paper proposed a new perspective on the rank-one bandit problem by showing it can be cast into the unimodal bandit framework. This led us to propose an algorithm closing the gap between existing regret upper and lower bound for Bernoulli rank-one bandits: Unimodal Thompson Sampling ($\uts$).
$\uts$ is easy to implement and very efficient in practice, as our experimental study reveals an improvement of a factor at least 20 with respect to the state-of-the art $\rankelimKL$ algorithm. Our main theoretical contribution is a novel regret analysis of this algorithm in the general unimodal setting, which sheds a new light on the leader exploration parameter to use. Interestingly, forcing exploration of the leader appears to help in practice in the rank-one example, and it may be interesting to investigate whether this remains the case for other structured bandit problems \citep{combes2017minimal}.

\paragraph{Acknowledgement} The authors acknowledge the French National Research Agency under projects BADASS (ANR-16-CE40-0002) and BOLD (ANR-19-CE23-0026-04).

\bibliography{mybib}

\newpage
\appendix

\section{Important Results} 

We recall two important results that are repeatedly used in our analysis. 

\begin{lemma} \label{hoeffding} (Hoeffding's inequality)
Let $X_1, ..., X_n$ be independent bounded random variables supported in $[0,1]$. For all $t \geq 0$,
$$\bP \left( \frac{1}{n} \sum_{i=1}^n (X_i - \bE[X_i]) \geq t \right) \leq \exp (-2nt^2)$$
and 
$$\bP \left( \frac{1}{n} \sum_{i=1}^n (X_i - \bE[X_i]) \leq -t \right) \leq \exp (-2nt^2)$$
\end{lemma}{}

\begin{lemma} (Beta Binomial trick) \label{Beta_Binomial trick}
Letting $F^{\emph{Beta}}_{\alpha,\beta}$ and $ F^{\emph{Bin}}_{n,p}$ respectively denote the cumulative distribution function of a Beta distribution with parameters $\alpha, \beta$, and of a Binomial distribution with parameters $(n,p)$. It holds that
$$ F^{\emph{Beta}}_{\alpha,\beta} (y) = 1 - F^{\emph{Bin}}_{\alpha + \beta -1, y} (\alpha - 1)$$
\end{lemma}

\section{Proof of Lemma~\ref{R_1_bound}} \label{appendix:R_1_bound}

In this section, we adapt the analysis of \cite{thompson_samplingAG}, highlighting the steps that need extra justification.

Let $k$ be a sub-optimal arm. We introduce two thresholds $x_k$ and $y_k$ such that $\mu_k < x_k < y_k < \mu_{k_\star}$, that we specify later. We define the following ``good" events: $E_k^\mu$ (t) = \{$\hat{\mu}_k(t) \leq x_k$\} and $E_k^\theta$ (t) = \{$\theta_k(t) \leq y_k$\}.
The event $\{K(t) = k, L(t) = k_\star\}$ can be decomposed as follows:
\begin{align}
    \{K(t) = k, L(t) = k_\star\} &= \{K(t) = k, L(t) = k_\star, E_k^\mu(t), E_k^\theta(t)\} \label{R_1_UTS1}\\ 
                            &\cup \{K(t) = k, L(t) = k_\star, E_k^\mu(t), \overline{E_k^\theta(t)}\} \label{R_1_UTS2} \\
                            &\cup \{K(t) = k, L(t) = k_\star, \overline{E_k^\mu(t)}\}  \label{R_1_UTS3}
\end{align}
Observe that for $k\notin \cN(k_\star)$, by definition of the algorithm, $\{K(t) = k, L(t) = k_\star\} =\emptyset$. For $k \in \cN(k_\star)$, we now upper bound the probability of the three events in the decomposition.

\paragraph{Upper Bound on the Probability of (\ref{R_1_UTS1})} We prove the following lemma.

\begin{lemma} \label{TS_1} For all $k \in \cN(k_\star)$, there exists a constant $\Bar{C}_1(\mu_{k_\star}, y_k)$ such that
$$\sum_{t=1}^T \mathbb{P}\left(K(t) = k, L(t) = k_\star, E_k^\mu(t), E_k^\theta(t)\right)
\leq \Bar{C}_1(\mu_{k_\star}, y_k)$$
\end{lemma}{}
\begin{proof}
We first prove the following inequality
\begin{small}\begin{equation}\label{TS_lemma_1}
    \mathbb{P}\left(K(t) = k, L(t) = k_\star, E_k^\mu(t), E_k^\theta(t)|\mathcal{F}_{t-1}\right)
\leq \frac{1-p_{kt}}{p_{kt}} \mathbb{P}\left(K(t) = k_\star, L(t) = k_\star, E_k^\mu(t), E_k^\theta(t)|\mathcal{F}_{t-1}\right)
\end{equation}\end{small}
\hspace{-0.3cm} where $p_{kt} = \mathbb{P} (\theta_1(t)>y_k|\cF_{t-1}) = \mathbb{P} (\overline{E_k^\theta(t)}|\cF_{t-1})$. To do so, notice that $E_k^\mu$ (t) and $\{L(t) = k_\star\}$ are $\mathcal{F}_{t-1}$-measurable, since $\hat{\mu}_k(t)$ is completely determined by the rewards and arms drawn up to time $t-1$. Therefore, one can assume that $\cF_{t-1}$ is such that $E_k^\mu$ (t) and $\{L(t) = k_\star\}$ hold, and it suffices to show that 
$$ \bP(K(t)=k, E_k^\theta (t) | \cF_{t-1}) \leq \frac{1-p_{kt}}{p_{kt}} \mathbb{P}(K(t) = k_\star,  E_k^\theta(t)|\mathcal{F}_{t-1})$$
which can be proved as in \cite{thompson_samplingAG}. With \eqref{TS_lemma_1}, we get
\begin{align*}
&\sum_{t=1}^T \mathbb{P}(K(t) = k, L(t) = k_\star, E_k^\mu(t), E_k^\theta(t)) \\
&= \sum_{t=1}^T \bE \left[ \mathbb{P}(K(t) = k, L(t) = k_\star, E_k^\mu(t), E_k^\theta(t) | \cF_{t-1}) \right]\\
&\leq \sum_{t=1}^T \bE \left[ \bE \left[ \frac{1-p_{kt}}{p_{kt}} \ind (K(t) = k_\star, L(t) = k_\star, E_k^\mu(t), E_k^\theta(t)|\mathcal{F}_{t-1} \right] \right]\\
&\leq \sum_{t=1}^T \bE \left[ \frac{1-p_{kt}}{p_{kt}} \ind (K(t) = k_\star, L(t) = k_\star, E_k^\mu(t), E_k^\theta(t) ) \right]\\
&\leq \sum_{t=1}^T \bE \left[ \frac{1-p_{kt}}{p_{kt}} \ind (K(t) = k_\star, E_k^\mu(t), E_k^\theta(t) ) \right]
\end{align*}
which allows to continue with the same proof as Theorem 1 in \cite{thompson_samplingAG}.
\end{proof}

\paragraph{Upper Bound on the Probability of (\ref{R_1_UTS2})} We prove the following lemma.
\begin{lemma} For all $k \in \cN(k_\star)$, letting $L_k(T) = \frac{\ln T}{\kl(x_k,y_k)}$, it holds that
$$ \sum_{t=1}^T \bP \left(K(t)=k, L(t) = k_\star, \overline{E_k^\theta(t)}, E_k^\mu(t) \right) \leq L_k(T) + 1.$$
\end{lemma}

\begin{proof} We start by the following decomposition:
\begin{align*}
\sum_{t=1}^T & \ \bP \left(K(t)=k, L(t) = k_\star, \overline{E_k^\theta(t)}, E_k^\mu(t) \right) \\
&\leq \bE \left[ \sum_{t=1}^T \ind \left(K(t)=k, L(t) = k_\star, N_k(t) \leq L_k(T), \overline{E_k^\theta(t)}, E_k^\mu(t) \right)\right]\\
&+ \bE \left[ \sum_{t=1}^T \ind \left(K(t)=k, L(t) = k_\star, N_k(t) > L_k(T), \overline{E_k^\theta(t)}, E_k^\mu(t) \right) \right]
\end{align*}
The first term of the sum is clearly bounded by $L_k(T)$.
As for the second term, we can directly upper bound it as follows
\begin{align*}
&\bE \left[ \sum_{t=1}^T \ind \left(K(t)=k, L(t) = k_\star, N_k(t) > L_k(T), \overline{E_k^\theta(t)}, E_k^\mu(t) \right) \right]\\
&\leq \bE \left[ \sum_{t=1}^T \ind \left(K(t)=k, N_k(t) > L_k(T), \overline{E_k^\theta(t)}, E_k^\mu(t) \right) \right]
\end{align*}
and the conclusion follows from the same steps used in the proof of Lemma 4 of \cite{thompson_samplingAG}.
\end{proof}

\paragraph{Upper Bound on the Probability of (\ref{R_1_UTS3})} We prove the following lemma.

\begin{lemma}\label{TS_3}
For $k \in \cN(k_\star)$,
$$ \sum_{t=1}^T \bP \left(K(t)=k, L(t)=k_\star, \overline{E_k^\mu(t)} \right) \leq \frac{1}{\kl(x_k,\mu_k)} + 1 $$ 
\end{lemma}

\begin{proof} To prove this lemma, one can write
\begin{align*}
    \sum_{t=1}^T \bP(K(t)=k, L(t)=k_\star, \overline{E_k^\mu(t)})
    &= \bE \left[\sum_{t=1}^T \ind ( K(t)=k, L(t)=k_\star, \overline{E_k^\mu(t)}  )\right]\\
    &\leq \bE \left[\sum_{t=1}^T \ind ( K(t)=k, \overline{E_k^\mu(t)}  )\right]\\
    \end{align*}
and use the same steps as in the proof of Lemma 3 of \cite{thompson_samplingAG}.
\end{proof}    
    
\paragraph{Conclusion}

For $0<\epsilon \leq 1$, we can choose $x_k$ and $y_k$ in $(\mu_k, \mu_{k_\star})$ such that 
$\kl(x_k,y_k) =\frac{\kl(\mu_k,\mu_{k_\star})}{(1+\epsilon)}$. Using the three above lemmas yields, for all $k \in \cN(k_\star)$:
\begin{align*}
 \bE \left[ \sum_{t=1}^T \ind (K(t)=k, L(t)=k_\star) \right] 
 &\leq  (1+\epsilon) \frac{\Delta_k}{\kl(\mu_k,\mu_\star)}\ln(T)+\Tilde{C}(\bm\mu,\epsilon).
\end{align*}
Since when the leader is $k_\star$, $\ind (K(t)=k, L(t)=k_\star) = 0$ for all $k \notin \cN^+(k_\star)$, we only need to sum over the arms $k \in \cN(k_\star)$ to get the result of Lemma \ref{R_1_bound}.

\newpage

\section{Proof of Lemma~\ref{lem:TermeRelou}}\label{appendix:TermeRelou}

Let $k \in [K] \setminus\{k_\star\}$.
\paragraph{Notation}
Recall from Section~\ref{sec:proof_uts} that $\cB_{\cN(k)} = \text{argmax}_{\ell\in \cN(k)} \mu_\ell$ is the set of best arms in the neighborhood of $k$. This set is such that $1 \leq |\cB_{\cN(k)}| \leq \tilde{B} $, and arms belonging to $\cB_{\cN(k)}$ have same mean, that we denote $\mu_{k_2} =  \max_{\ell\in \cN(k)} \mu_\ell$. We also define $N_{\cB_{\cN(k)}}(t) = \sum_{k_2 \in \cB_{\cN(k)}} N_{k_2}(t)$, the number of times arms belonging to $\cB_{\cN(k)}$ have been drawn up to time $t$. We will say that $k' \in \cN^+(k)$ is sub-optimal if $\mu_{k'} < \mu_{k_2}$. We denote by $\Tilde{M}_k = |\cN^+(k) \setminus \cB_{\cN(k)}| \leq |\cN(k)|$, the number of sub-optimal arms belonging to $\cN(k)$.

We introduce, for every arm $k'$,
\[\delta_{k'} = \frac{\mu_{k_2} - \mu_{k'}}{2}, \ \ \ \text{and let} \ \ \ \delta = \underset{k' \in \cN^+(k) \setminus \cB_{\cN(k)}}{\text{min}} \delta_{k'} \ \ \ \text{and} \ \ \ \ C := \frac{6}{\delta^2}.\]
We denote by $\Tilde{k}$ any arm satisfying $\delta_{\Tilde{k}} = \delta$.

Just like in Section~\ref{sec:analysis}, we introduce the consecutive instants in which arm $k$ is the leader, $\tau_i^k$. Assuming that $\uts$($\gamma$) would be played forever, the instant of the $i$-th time arm $k$ is the leader, $\tau_i^k$, can be formally written as such
\[\tau_{i}^k = \inf \{ t \in \N : L(t) = k, \ell_k(t) = i \},\]
with the convention that $\inf \emptyset = +\infty$.

\medskip

For all $i  \in \{1,\dots,T\}$, for all $b\in (0,1)$, by definition of $\tau_i^k$, it holds that
\[\sum_{t=1}^T\ind\left(L(t) = k, \ell_k(t) = i, \forall k_2 \in \cB_{\cN(k)}, N_{k_2}(t) \leq \left(\ell_k(t)\right)^b\right) =\ind \left(\forall k_2 \in \cB_{\cN(k)}, N_{k_2}(\tau_{i}^k) \leq i^b\right)\ind\left(\tau_i^k \leq T\right),\]
which permits to rewrite
\begin{align}
\sum_{t=1}^T \bP\left(L(t) = k, \forall k_2 \in \cB_{\cN(k)}, N_{k_2}(t) \leq (\ell_k(t))^b\right)
&= \sum_{i=1}^T \bP\left(\forall k_2 \in \cB_{\cN(k)}, N_{k_2}\left(\tau_{i}^k\right) \leq i^b, \tau_i^k \leq T\right) \nonumber \\
&\leq \sum_{i=1}^T \bP\left(N_{\cB_{\cN(k)}}\left(\tau_{i}^k\right) \leq \tilde{B} i^b, \tau_i^k \leq T\right), \label{ToControl}
\end{align}
where we recall that $N_{\cB_{\cN(k)}}(t)$ is the total number of pulls of all arms in $B_{\cN}(k)$.

We now provide an upper bound on \eqref{ToControl}, for a well chosen value of $b$.

Our analysis bears similarity with that of \cite{ALT12}: we use the fact that if arms belonging to $ \cB_{\cN(k)}$ are not drawn much at time $\tau_i^k$, there must exist many consecutive instants $\tau_{\ell}^k < \tau_i^k$ in which those arms are not selected at all. To formalize this idea, we introduce for every pair $i,j$ the first instant preceeding $\tau_i^k$ in which arms of $ \cB_{\cN(k)}$ have been played at least $j$ times while arm $k$ is the leader:
\[\nu_{i,j} = \inf\{ \ell \leq i : N_{\cB_{\cN(k)}}(\tau_\ell^k) \geq j\},\]
with the convention $\inf \emptyset = i+1$. It holds that
\[\left(N_{\cB_{\cN(k)}}\left(\tau_{i}^k\right) \leq \tilde{B} i^b\right) = \left(\nu_{i,\lceil \tilde{B} i^b\rceil} = i + 1\right)\subseteq \bigcup_{j=0}^{\lfloor \tilde{B} i^b\rfloor} \left(\nu_{i,j+1} - \nu_{i,j} \geq \frac{i^{1-b}}{\tilde{B}}-1\right). \]

We now introduce $\cI_{i,j} \subseteq \left(\nu_{i,j} , \nu_{i,j} + \lceil \frac{i^{1-b}}{\tilde{B}} - 2 \rceil \right]$, the subset of instants belonging to $\left(\nu_{i,j} , \nu_{i,j} + \lceil \frac{i^{1-b}}{\tilde{B}} - 2 \rceil \right]$ where no leader exploration is performed. The $j$-th event in the union implies that no arm belonging to $ \cB_{\cN(k)}$ is selected in any instant $\tau_\ell^k$ for $\ell \in \cI_{i,j}$. More precisely, introducing
$$\cE_{i,j} =  \left\{\cI_{i,j} \subseteq [i]\right\}\\
\bigcap\left\{\forall \ell \in \cI_{i,j}, K(\tau_\ell^k) \notin \cB_{\cN(k)} \right\}$$
one has \begin{equation}\label{Useful}\bP\left(N_{\cB_{\cN(k)}}\left(\tau_i^k\right) \leq i^b, \tau_i^k \leq T\right) \leq \sum_{j=0}^{\lfloor \tilde{B}i^b\rfloor} \bP\left(\cE_{i,j}, \tau_i^k \leq T\right). \end{equation}

\paragraph{Interval sub-division and saturated arms} To further upper bound \eqref{Useful}, we introduce for $m = 1, \dots, \tilde{M}_k + 1$, the intervals $\cI_{i,j,m}$:
$$\cI_{i,j,m} := \left(\nu_{i,j} + (m-1) \Bigl\lfloor \frac{i^{1-b}/\tilde{B} - 2}{\tilde{M}_k + 1} \Bigr\rfloor, \nu_{i,j} + m \Bigl\lfloor\frac{i^{1-b}/\tilde{B} - 2}{\tilde{M}_k + 1} \Bigr\rfloor \right] \cap \cI_{i,j}, $$
whose length is lower bounded as follows, substracting the instant in which leader exploration is performed (that are not included in $\cI_{i,j}$): \[|\cI_{i,j,m}| = \Bigl\lfloor \frac{i^{1-b}/\tilde{B} - 2}{\tilde{M}_k + 1} \Bigr\rfloor - \Bigl\lceil \frac{1}{\gamma} \left(\frac{i^{1-b}/\tilde{B} - 2}{\tilde{M}_k + 1}\right) \Bigr\rceil  \geq \Bigl\lfloor \left(1-\frac{1}{\gamma}\right) \left(\frac{i^{1-b}/\tilde{B} - 2}{\tilde{M}_k + 1}\right) - 2 \Bigr\rfloor :=\tilde{H}_{i,b,k,\gamma}.\]

As in \cite{ALT12}, we introduce the notion of saturated sub-optimal arm: we say an arm $k'\notin \cB_{\cN(k)}$ is saturated at $\ell$ if $N_{k'}(\tau_\ell^k) > C \ln(i)$. Otherwise, it is unsaturated. For an interval $\cI_{i,j,m}$, we denote by $n_{i,j,m}$ the number of interruptions, that is, the number of times we draw an unsaturated arm during $\cI_{i,j,m}$. We introduce $F_{i,j,m}$, the event that by the end of $\cI_{i,j,m}$ , at least $m$ sub-optimal arms are saturated, and $\cS_{i,j,m}$, the set of saturated arms at the end of $\cI_{i,j,m}$.


We decompose the probability of the event $\{\cE_{i,j}, \tau_i^k \leq T \}$ as follows

\begin{align}
\mathbb{P} [\cE_{i,j}, \tau_i^k \leq T] &\leq \mathbb{P} [\cE_{i,j}, F_{i,j,\tilde{M}_k}, \tau_i^k \leq T] \label{step_2}\\
&+ \mathbb{P} [\cE_{i,j}, F_{i,j,\tilde{M}_k}^c,\tau_i^k \leq T] \label{step_3}
\end{align}

We will prove below that
\begin{equation} \label{step_2_bound}
     \eqref{step_2} \leq \frac{2\tilde{M}_k}{i^2(1-\exp(- \delta^2/2))} + g_1(\bm\mu, j, b, i, k, \gamma)
\end{equation}
and that for $i$ larger than some constant $N_{ \bm\mu,b}$,
\begin{equation} \label{step_3_bound}
     \eqref{step_3} \leq (\tilde{M}_k -1) \left( \frac{2\tilde{M}_k}{i^2(1-\exp(- \delta^2/2))} + g_2(\bm\mu, j, b, i, k, \gamma) \right)
\end{equation}
where for a well-chosen $b\in (0,1)$ and $\gamma \geq 2 $
 \[\sum_{i=1}^\infty\sum_{j=0}^{\lfloor \tilde{B} i^b \rfloor} g_1(\bm\mu, j, b, i, k, \gamma) < \infty  \ \text{ and } \ \sum_{i=1}^{\infty}\sum_{j=0}^{\lfloor \tilde{B} i^b \rfloor} g_2(\bm\mu, j, b, i, k, \gamma) < \infty.\]
Combining \eqref{Useful} with the upper bounds \eqref{step_2_bound} and \eqref{step_3_bound}, we get
\begin{align*}
\eqref{ToControl}& \leq M_{\bm\mu,b} + \sum_{i=N_{\bm\mu,b}+1}^T \sum_{j=0}^{\lfloor \tilde{B}i^b \rfloor} \bP [\cE_{i,j}, \tau_i^k \leq T] \\
&\leq M_{\bm\mu,b} + \sum_{i=1}^T \sum_{j=0}^{\lfloor \tilde{B}i^b \rfloor} \left[ \frac{2\tilde{M}_k}{i^2(1-\exp(- \delta^2/2))} + g_1(\bm\mu, j, b, i,k, \gamma) \right]\\
&+ \sum_{i=1}^T \sum_{j=0}^{\lfloor \tilde{B}i^b \rfloor} \left[ (\tilde{M}_k-1) \left(\frac{2\tilde{M}_k}{i^2(1-\exp(- \delta^2/2))} + g_2(\bm\mu, j, b, i,k, \gamma)\right) \right] \\
&\leq M_{\bm\mu,b} + \sum_{i=1}^\infty \frac{2\tilde{B}\tilde{M}_k^2}{i^{2-b}(1-\exp(- \delta^2/2))} + \sum_{i=1}^\infty \sum_{j=0}^{\lfloor \tilde{B} i^b \rfloor} \left[g_1(\bm\mu, j, b, i,k,\gamma) + g_2(\bm\mu, j, b, i,k,\gamma)\right] \\
&:= D_k(\bm\mu,b,\gamma),
\end{align*}{}
which concludes the proof. We now prove the two crucial upper bounds \eqref{step_2_bound} and \eqref{step_3_bound}.

\paragraph{Main ingredients} We introduce two useful lemmas whose proofs are postponed to the end of this appendix.
Lemma \ref{lemma4} establishes that it is unlikely that the Thompson sample associated to some saturated arm exceeds its true mean by too much.
\begin{lemma} \label{lemma4}
Let $k \in [K]$.
$$\mathbb{P} \left(\exists \ell \leq i , \exists k' \notin \cB_{\cN(k)}, \theta_{k'}(\tau_\ell^k) > \mu_{k'} + \delta, N_{k'}( \tau_\ell^k) > C \ln(i),\tau_{i}^k < T\right) \leq \frac{2\tilde{M}_k}{i^2(1-\exp(- \delta^2/2))}$$
\end{lemma}
Lemma \ref{lemma3} shows that the Thompson samples of an arm belonging to $\cB_{\cN(k)}$ are unlikely to fall below $\mu_{\Tilde{k}} + \delta$ during a long interval in which the posterior of this arm doesn't evolve.

\begin{lemma} \label{lemma3}
Let $\tilde{\cI}$ be a random interval such that $\forall \ell \in \Tilde{\cI}, N_{\cB_{\cN(k)}}(\tau_\ell^k) = j$ and $|\Tilde{\cI}| \geq x$ for some deterministic constant $x$. There exists $\lambda_0 = \lambda_0(\mu_{k_2}, \mu_{\Tilde{k}},\delta) > 1$ such that for $\lambda \in ]1,\lambda_0[$,
$$\bP \left(\forall \ell \in \Tilde{\cI}, \forall k_2 \in \cB_{\cN(k)}, \theta_{k_2}(\tau_\ell^k) \leq \mu_{\Tilde{k}} + \delta \right) \leq j \tilde{B} (\alpha_{\mu_{\tilde{k}},\delta})^x +  C_{\lambda,\mu_{k_2},\mu_{\tilde{k}}}\frac{\exp(-jd_{\lambda,\mu_{k_2},\mu_{\tilde{k}}}/\tilde{B})}{x^\lambda},$$
where $C_{\lambda,\mu_{k_2},\mu_{\Tilde{k}}},d_{\lambda,\mu_{k_2},\mu_{\Tilde{k}}} > 0 $, and $\alpha_{\mu_{\Tilde{k}},\delta} = \left(\frac{1}{2}\right)^{1-\mu_{\Tilde{k}}-\delta}$.
\end{lemma}

\paragraph{Proof of the Upper bound \eqref{step_2_bound}}
On the event $\cE_{i,j} \cap F_{i,j,\tilde{M}_k}$, only saturated arms are drawn during the interval $\cI_{i,j,\tilde{M}_k+1}$, so that one has the following decomposition:
\begin{align*}
&\mathbb{P}[\cE_{i,j} \cap F_{i,j,\tilde{M}_k}\cap \{\tau_i^k \leq T\}] \\
&\leq \mathbb{P}[\{\exists \ell \in \cI_{i,j,\tilde{M}_k+1}, \exists k' \notin \cB_{\cN(k)}, \theta_{k'}(\tau_\ell^k) > \mu_{k'} + \delta \} \cap \{ N_{k'} (\tau_\ell^k) > C \ln(i) \}\cap \cE_{i,j} \cap \{\tau_i^k \leq T\}] \\
& \hspace{2mm}+  \mathbb{P}[\{\forall \ell \in \cI_{i,j,\tilde{M}_k+1}, \forall k' \notin \cB_{\cN(k)}, \theta_{k'}(\tau_\ell^k) \leq \mu_{k'} + \delta\} \cap \cE_{i,j} \cap F_{i,j,|\cN^+(k)|-1} ]\\
&\leq \mathbb{P}\left(\exists \ell \leq i, \exists k' \notin \cB_{\cN(k)}, \theta_{k'}(\tau_\ell^k) > \mu_{k'} + \delta, N_{k'} (\tau_\ell^k) > C \ln(i) , \tau_{i}^k \leq T\right) \\
&\hspace{2mm}+  \mathbb{P}({\forall \ell \in \cI_{i,j,\tilde{M}_k+1}, \forall k_2 \in \cB_{\cN(k)}, \theta_{k_2}(\tau_\ell^k) \leq \mu_{\Tilde{k}} + \delta} , \cE_{i,j})
\end{align*}

Using Lemma~\ref{lemma4}, we can bound the first term in this sum by
\begin{align*}\frac{2\tilde{M}_k}{i^2(1-\exp(- \delta^2/2))}.
\end{align*}

On the event $\cE_{i,j}$, $N_{\cB_{\cN(k)}}(\tau_\ell^k) = j$ for all $\ell \in \cI_{i,j,\tilde{M}_k+1}$. Lemma~\ref{lemma3} with $\Tilde{\cI}= \cI_{i,j,\tilde{M}_k+1}$ and $ x=\tilde{H}_{i,b,k,\gamma}$ yields the following upper bound for the second term

\[
 j \tilde{B} (\alpha_{\mu_{\tilde{k}},\delta})^{\tilde{H}_{i,b,k,\gamma}} +  C_{\lambda,\mu_{k_2},\mu_{\tilde{k}}}\frac{\exp(-jd_{\lambda,\mu_{k_2},\mu_{\tilde{k}}}/\tilde{B})}{\tilde{H}_{i,b,k,\gamma}^\lambda} := g_1(\bm\mu, j, b, i, k,\gamma).
\]
Summing $  g_1(\bm\mu, j, b, i, k,\gamma)$ over $j\leq\lfloor \tilde{B}i^b \rfloor$ and expliciting $\tilde{H}_{i,b,k,\gamma}$ gives
\begin{align*}
\sum_{j\leq\lfloor \tilde{B}i^b \rfloor}  g_1(\bm\mu, j, b, i, k, \gamma)
= \tilde{B} \frac{\lfloor \tilde{B}i^b\rfloor(\lfloor \tilde{B}i^b\rfloor+1)}{2} (\alpha_{\mu_{\tilde{k}},\delta})^{\Bigl\lfloor \left(1-\frac{1}{\gamma}\right) \frac{i^{1-b}/\tilde{B} - 2}{\tilde{M}_k + 1} - 2\Bigr\rfloor} +  \frac{C'_{\lambda,\mu_{k_2},\mu_{\tilde{k}}} }{ \Bigl\lfloor \left(1-\frac{1}{\gamma}\right) \frac{i^{1-b}/\tilde{B} - 2}{\tilde{M}_k + 1} -2 \Bigr\rfloor^\lambda}.
\end{align*}
The first term of the sum is $o\left(\frac{1}{i^2}\right)$, and by choosing $b < 1 - \frac{1}{\lambda}$ for the second term, we obtain that $\sum_{i\leq\infty} \sum_{j\leq\lfloor i^b \rfloor} g_1(\bm\mu, j, b, i, k, \gamma)$ is finite when $\gamma>1$.

\paragraph{Proof of the Upper Bound \eqref{step_3_bound}}
Similarly to \cite{ALT12}, we prove by induction that for all $ 2 \leq m \leq \tilde{M}_k+1$, if $i$ is larger than some deterministic constant $N_{\bm\mu, b}$,
$$\mathbb{P} [\cE_{i,j} \cap F_{i,j, m -1}^c \cap \{\tau_i^k \leq T\}] \leq (m-2)\left( \frac{2\tilde{M}_k}{i^2(1-\exp(- \delta^2/2))} + g_2(\bm\mu, j, b, i,k,\gamma) \right),$$
where $N_{\bm\mu,b}$ and $g_2(\bm\mu,j,b,i,k,\gamma)$ are made precise below.

\medskip

\noindent\underline{Base case of the induction}: On the event $\cE_{i,j}$, only suboptimal arms are played during the interval $\cI_{i,j,1}$, of length larger than $\tilde{H}_{i,b,k,\gamma}$. Hence at least one suboptimal arm must be played more than $\lceil  \frac{\tilde{H}_{i,b,k,\gamma}}{\tilde{M}_k} \rceil$ times. Besides, there exists some deterministic constant $N_{\bm\mu, b}$ such that for $i > N_{\bm\mu, b}$, $\lceil  \frac{\tilde{H}_{i,b,k,\gamma}}{\tilde{M}_k} \rceil \geq C \ln(i)$.

Therefore, when $i \geq N_{\bm\mu, b}$, at least one suboptimal arm is saturated by the end of $\cI_{i,j,1}$, so that for $i \geq N_{\bm\mu, b}$, $\mathbb{P} (\cE_{i,j} \cap F_{i,j, 1}^c \cap \{\tau_i^k \leq T\}) = 0$. Hence, the inequality holds for $m=2$.

\medskip

\noindent\underline{Induction}: Let us assume the following, for some $m  \in \{2,\dots,\tilde{M}_k\}$:
$$\mathbb{P} (\cE_{i,j} \cap F_{i,j,m-1}^c \cap \{\tau_i^k \leq T\}) \leq (m-2) \left(\frac{2\tilde{M}_k}{i^2(1-\exp(- \delta^2/2))} + g_2(\bm\mu, j, b, i,k.\gamma) \right)\;. $$
Exploiting this inductive hypothesis, one obtains
\begin{align*}
    &\mathbb{P} (\cE_{i,j} \cap F_{i,j, m}^c\cap \{\tau_i^k \leq T\}) \\ &\hspace{0.5cm}\leq \mathbb{P} (\cE_{i,j} \cap F_{i,j,m-1}^c\cap \{\tau_i^k \leq T\}) + \bP(\cE_{i,j} \cap F_{i,j, m}^c \cap F_{i,j, m-1}\cap \{\tau_i^k \leq T\}) \\
    &\hspace{0.5cm}\leq (m-2) \left(\frac{2\tilde{M}_k}{i^2(1-\exp(- \delta^2/2))} + g_2(\bm\mu, j, b, i,k, \gamma) \right) + \bP(\cE_{i,j} \cap F_{i,j, m}^c \cap F_{j, m-1} \cap \{\tau_i^k \leq T\})\;.
\end{align*}

Let us prove that the second term of the sum is bounded by $ \frac{2\tilde{M}_k}{i^2(1-\exp(- \delta^2/2))} + g_2(\bm\mu, j, b, i,k)$.

On the event $(\cE_{i,j} \cap F_{i,j,m}^c \cap F_{i,j,m-1})$, there are exactly $m-1$ saturated arms at the beginning of interval $\cI_{i,j,m}$ and no new arm is saturated during this interval, so that $\cS_{i,j,m-1} = \cS_{i,j,m}$. As a result, there cannot be more than $\tilde{M}_kC\ln(i)$ interruptions during this interval, so that
\begin{align}
&\mathbb{P}(\cE_{i,j} \cap F_{i,j,m}^c \cap F_{i,j,m-1} \cap \{\tau_i^k \leq T\}) \nonumber \\
&\leq \bP  (\cE_{i,j} \cap F_{i,j,m-1} \cap \{n_{i,j,m} \leq \tilde{M}_kC\ln(i) \}\cap \{\tau_i^k \leq T\}) \nonumber \\
&\leq \bP (\{ \exists \ell \in \cI_{i,j,m}, \exists k' \in \cS_{i,j,m-1} \backslash \cB_{\cN(k)}, \theta_{k'}(\tau_\ell^k) > \mu_{k'} + \delta \} \cap \cE_{i,j} \cap \{\tau_i^k \leq T\})\label{induction_termA} \\
&+ \bP (\{ \forall \ell \in \cI_{i,j,m}, \forall k' \in \cS_{i,j,m-1}\backslash \cB_{\cN(k)}, \theta_{k'}(\tau_\ell^k) \leq \mu_{k'} + \delta \} \cap \cE_{i,j} \cap F_{i,j,m-1} \cap \{n_{i,j,m} \leq \tilde{M}_k C \ln(i) \}) \label{induction_termB}
\end{align}

Lemma \ref{lemma4} allows us to bound the term \eqref{induction_termA}:
$$ \eqref{induction_termA} \leq \frac{2\tilde{M}_k}{i^2(1-\exp(- \delta^2/2))}.$$

To deal with \eqref{induction_termB}, we introduce the random intervals  \[\cJ_h = \{\ell \in \cI_{i,j,m}, \text{between the }h\text{-th and }(h+1)\text{-th interrruptions}\}.\]
On the event in the probability of \eqref{induction_termB}, there exists an interval $\cJ_h$ of length larger than $\lceil \frac{ \tilde{H}_{i,b,k,\gamma}}{\tilde{M}_k C\ln(i)} \rceil$ such that there is no interruption at times $\tau_\ell^k$, for $\ell \in \cJ_h$. This means that, at these time steps, all Thompson samples are smaller than that of the greatest sample among the saturated arms (which are themselves smaller than $\mu_{\Tilde{k}} + \delta$). In particular, in this interval, $\forall k_2 \in \cB_{\cN(k)}, \theta_{k_2}(\tau_\ell^k) \leq \mu_{\Tilde{k}} + \delta$, and we get
\begin{align*}
\eqref{induction_termB}
&\leq \bP \left( \{ \exists h \in {0,...,n_{i,j,m}-1}, |\cJ_h| \geq \Bigl \lceil\frac{ \tilde{H}_{i,b,k,\gamma}}{\tilde{M}_k C\ln(i)} \Bigr\rceil \} \right.\\
& \hspace{2cm}\left.\cap \{\forall \ell \in \cJ_h, \forall k' \in \cS_{i,j,m-1} \backslash \cB_\cN(k), \theta_{k'}(\tau_\ell^k) \leq \mu_{\Tilde{k}} + \delta \}  \cap \cE_{i,j} \cap F_{i,j,m-1}  \right)\\
&\leq \sum_{h=0}^{\tilde{M}_kC\ln(i)} \bP \left( \{ |\cJ_h| \geq  \Bigl\lceil \frac{ \tilde{H}_{i,b,k,\gamma}}{\tilde{M}_k C\ln(i)}\Bigr\rceil \} \cap \{ \forall \ell \in \cJ_h, \forall k_2 \in \cB_{\cN(k)}, \theta_{k_2}(\tau_\ell^k) \leq \mu_{\Tilde{k}} + \delta \} \cap \cE_{i,j}\right).
\end{align*}

Applying Lemma \ref{lemma3} with $\tilde{\cI} = \cJ_h$, we get
\begin{align*}
    \eqref{induction_termB}
    &\leq  \tilde{M}_kC\ln(i) \left[j \tilde{B} (\alpha_{\mu_{\tilde{k}},\delta})^{\Bigl\lceil \frac{(1-1/\gamma) (i^{1-b}/\tilde{B} -2) -2(\tilde{M}_k+1)}{\tilde{M}_k(\tilde{M}_k+1) C\ln(i)}\Bigr\rceil} +  C_{\lambda,\mu_{k_2},\mu_{\tilde{k}}}\frac{\exp(-jd_{\lambda,\mu_{k_2},\mu_{\tilde{k}}}/\tilde{B})}{\Bigl\lceil \frac{(1-1/\gamma) (i^{1-b}/\tilde{B} -2)-2(\tilde{M}_k+1)}{\tilde{M}_k(\tilde{M}_k+1) C\ln(i)} \Bigr\rceil^\lambda} \right]\\
    &:= g_2(\bm\mu, j, b, i,k,\gamma).
\end{align*}
This proves that
\[\mathbb{P} (\cE_{i,j} \cap F_{i,j,m}^c \cap \{\tau_i^k \leq T\}) \leq (m-1) \left(\frac{2\tilde{M}_k}{i^2(1-\exp(- \delta^2/2))} + g_2(\bm\mu, j, b, i,k,\gamma) \right)
\]
and the induction is verified.

As for $g_1(\bm\mu, j, b, i, k,\gamma)$, we observe that when $\gamma>1$, $\sum_{i\leq\infty} \sum_{j\leq\lfloor \tilde{B}i^b \rfloor} g_2(\bm\mu, j, b, i,k,\gamma)$ is finite by choosing $b < 1 - \frac{1}{\lambda}$.

\paragraph{Proof of Lemma \ref{lemma4}}

It holds that
\begin{align*}
&\mathbb{P} \left(\exists \ell\leq i, \exists k' \notin \cB_{\cN(k)}, \theta_{k'}(\tau_\ell^k) > \mu_{k'} + \delta, N_{k'}( \tau_\ell^k) > C \ln(i), \tau_i^k \leq T \right) \\
&\leq \sum_{\ell=1}^i \sum_{k' \in \cN^+(k) \backslash \cB_{\cN(k)}} \mathbb{P} \left(\theta_{k'}(\tau_\ell^k) > \mu_{k'} + \delta, N_{k'}( \tau_\ell^k) > C \ln(i), \tau_\ell^k \leq T\right)
\end{align*}
Let $\ell \leq i$, $k' \in \cN^+(k)\setminus \cB_{\cN(k)}$.
\begin{align}
&\mathbb{P} \left(\theta_{k'}(\tau_\ell^k) > \mu_{k'} + \delta, N_{k'}( \tau_\ell^k) > C \ln(i), \tau_\ell^k \leq T \right) \\
&\leq \bP \left(\hat{\mu}_{k'}(\tau_\ell^k) > \mu_{k'} + \delta/2, N_{k'} (\tau_\ell^k) > C \ln(i), \tau_\ell^k \leq T \right) \label{TA} \\
&+ \bP \left(\hat{\mu}_{k'}(\tau_\ell^k) \leq \mu_{k'} + \delta/2, \theta_{k'}(\tau_\ell^k) > \mu_{k'} + \delta, N_{k'} (\tau_\ell^k) > C \ln(i), \tau_\ell^k \leq T \right)  \label{TB}
\end{align}

Using a union bound over the values of $N_{k'} (\tau_\ell^k) \geq C\ln(i)$ together with Hoeffding's inequality (Lemma~\ref{hoeffding}) yields
\begin{align*}
\eqref{TA}
&\leq \sum_{u=C\ln(i)}^T \bP(\hat{\mu}_{k',u} > \mu_{k'} + \delta/2) \leq \sum_{u=C\ln(i)}^\infty \exp\left(-\frac{\delta^2 u}{2}\right) = \frac{\exp(- C\ln(i)\delta^2/2)}{1-\exp(- \delta^2/2)},
\end{align*}
where we denote by $\hat{\mu}_{k',u}$ the estimated mean of the $k'$-th arm at the $u$-th draw.

We upper bound \eqref{TB} by
\begin{align*}
&\sum_{u=C\ln(i)}^T \bP\left(\hat{\mu}_{k'}(\tau_\ell^k)\leq \mu_{k'} + \delta/2, \theta_{k'}(\tau_\ell^k) \geq \mu_{k'} + \delta, N_{k'}(\tau_\ell^k) = u , \tau_\ell^k \leq T\right)\\
&\leq\sum_{u=C\ln(i)}^T \bP\left(\mu_{k'} \geq \hat{\mu}_{k',u} - \delta/2, \theta_{k'}(\tau_\ell^k) \geq \mu_{k'} + \delta, N_{k'}(\tau_\ell^k) = u , \tau_\ell^k \leq T\right)\\
&\leq \sum_{u=C\ln(i)}^T \bP\left(\theta_{k'}(\tau_\ell^k) \geq \hat{\mu}_{k',u} + \delta/2 , N_{k'}(\tau_\ell^k) = u, \tau_\ell^k \leq T\right)\\
&\leq \bE\left[ \sum_{u=C\ln(i)}^T  \left(1 - F^{\text{Beta}}_{u\hat{\mu}_{k',u}+1, u-u\hat{\mu}_{k',u}+1}(\hat{\mu}_{k',u}+\delta/2) \right) \right]\\
&= \bE\left[ \sum_{u=C\ln(i)}^T  F^{\text{Bin}}_{u+1, \hat{\mu}_{k',u} + \delta/2} (u\hat{\mu}_{k',u}) \right] \\
&\leq \bE\left[ \sum_{u=C\ln(i)}^T F^{\text{Bin}}_{u, \hat{\mu}_{k',u} + \delta/2} (u\hat{\mu}_{k',u}) \right] \\
&\leq \bE\left[ \sum_{u=C\ln(i)}^\infty \exp(-u\delta^2/2) \right] \\
&= \frac{\exp(-C\ln(i) \delta^2/2)}{1 - \exp(-\delta^2/2)},
\end{align*}
where the first equality comes from the Beta-Binomial trick (Lemma \ref{Beta_Binomial trick}), and the last inequality comes from Hoeffding's inequality.

Combining \eqref{TA} and \eqref{TB}, and recalling that $C =6/\delta^2$, we get
\begin{align*}
&\mathbb{P} \left(\exists \ell\leq i, \exists k' \notin \cB_{\cN(k)}, \theta_{k'}(\tau_\ell^k) > \mu_{k'} + \delta, N_{k'}( \tau_\ell^k) > C \ln(i), \tau_i^k \leq T \right) \\
&\leq \sum_{\ell=1}^i \sum_{k' \in \cN^+(k)  \setminus\cB_{\cN(k)}} 2\frac{\exp(- C\ln(i)\delta^2/2)}{1-\exp(- \delta^2/2)}\\
&\leq \frac{2\tilde{M}_k}{i^{C\delta^2/2-1}(1-\exp(- \delta^2/2))}
 =  \frac{2\tilde{M}_k}{i^2(1-\exp(- \delta^2/2))}.
\end{align*}

\paragraph{Proof of Lemma \ref{lemma3}}
The interval $\Tilde{\cI}$ is such that for all $\ell \in \Tilde{\cI}$, $N_{\cB_{\cN(k)}}(\tau_\ell^k) = j$. This implies that there exists $k_2 \in \cB_{\cN(k)}$ which has been drawn at least $\frac{j}{\Tilde{B}}$ and is not drawn during that interval. Hence,

\begin{align}
&\bP \left( \forall \ell \in \Tilde{\cI}, N_{\cB_{\cN(k)}}(\tau_\ell^k) = j, \forall k_2 \in \cB_{\cN(k)}, \theta_{k_2}(\tau_\ell^k) \leq \mu_{\Tilde{k}} + \delta \right) \nonumber \\
&\leq \bP \left( \forall \ell \in \Tilde{\cI}, \exists k_2 \in \cB_{\cN(k)}, \frac{j}{\tilde{B}} \leq N_{k_2}(\tau_\ell^k) \leq j, \theta_{k_2}(\tau_\ell^k) \leq \mu_{\Tilde{k}} + \delta \right) \nonumber \\
&\leq \sum_{k_2 \in \cB_{\cN(k)}} \sum_{j_{k_2}=\frac{j}{\tilde{B}}}^j \bP \left( \forall \ell \in \Tilde{\cI}, N_{k_2}(\tau_\ell^k) = j_{k_2}, \theta_{k_2}(\tau_\ell^k) \leq \mu_{\Tilde{k}} + \delta \right) \label{double_sum}
\end{align}

If $ N_{k_2}(\tau_\ell^k) = j_{k_2}$ for all $\ell \in \Tilde{\cI}$, conditioned on $S_{k_2,j_{k_2}}$ (sum of first $j_{k_2}$ observations from arm $k_2$), the Thompson samples of arm $k_2$ drawn during this interval are an i.i.d. sequence with distribution Beta$(S_{k_2,j_{k_2}} + 1, j_{k_2} -S_{k_2,j_{k_2}} + 1)$. Therefore,
\begin{align*}
\bP \left( \forall \ell \in \Tilde{\cI}, N_{k_2}(\tau_\ell^k) = j_{k_2}, \theta_{k_2}(\tau_\ell^k) \leq \mu_{\Tilde{k}} + \delta  | S_{k_2,j_{k_2}} \right)
&= \left(F^{\text{Beta}}_{S_{k_2,j_{k_2}} + 1, j_{k_2} -S_{k_2,j_{k_2}} + 1} (\mu_{\Tilde{k}} + \delta) \right)^{|\tilde{\cI}|} \\
&\leq \left(F^{\text{Beta}}_{S_{k_2,j_{k_2}} + 1, j_{k_2} -S_{k_2,j_{k_2}} + 1} (\mu_{\Tilde{k}} + \delta) \right)^{x}\\
&= \left( 1-F^{\text{Bin}}_{j_{k_2}+1,\mu_{\Tilde{k}} + \delta}(S_{k_2,j_{k_2}}) \right)^{x}\\
\end{align*}
where the inequality holds because $|\tilde{\cI}| \geq x$, and the last equality is obtained by using the Beta-Binomial trick (Lemma~\ref{Beta_Binomial trick}).

It follows that
\begin{align*}
\bP \left( \forall \ell \in \Tilde{\cI}, N_{k_2}(\tau_\ell^k) = j_{k_2}, \theta_{k_2}(\tau_\ell^k) \leq \mu_{\Tilde{k}} + \delta \right)
&= \bE \left[ \bP \left( \forall \ell \in \Tilde{\cI}, N_{k_2}(\tau_\ell^k) = j_{k_2}, \theta_{k_2}(\tau_\ell^k) \leq \mu_{\Tilde{k}} + \delta  | S_{k_2,j_{k_2}} \right) \right]\\
&\leq \bE \left[ \left( 1-F^{\text{Bin}}_{j_{k_2}+1,\mu_{\Tilde{k}} + \delta}(S_{k_2,j_{k_2}}) \right)^{x} \right]
\end{align*}
where the expectation is taken with respect to $S_{k_2,j_{k_2}} \sim \text{Bin}(j_{k_2},\mu_{k_2})$.

An upper bound on this expectation is provided by the following lemma that can be extracted from the proof of Lemma 3 in \cite{ALT12}.
\begin{lemma} Let X be a random variable with Binomial distribution of parameter $(j,\mu_1)$. Let $\delta$ and $\mu_2$ be such that $0 < \mu_2 + \delta < \mu_1$. There exists $\lambda_0 = \lambda_0(\mu_1,\mu_2,\delta) >1$ such that for $\lambda \in (1,\lambda_0)$,
\[\bE\left[\left(1 - F^{\text{Bin}}_{j+1,\mu_2+\delta}(X)\right)^x\right]
\leq (\alpha_{\mu_{2},\delta})^{x} + C_{\lambda,\mu_1,\mu_2} \frac{\exp(-jd_{\lambda,\mu_1,\mu_2})}{x^\lambda}\]
where $C_{\lambda,\mu_1,\mu_2},d_{\lambda,\mu_1,\mu_2} > 0 $, and $\alpha_{\mu_{2},\delta} = \left(\frac{1}{2}\right)^{1-\mu_2-\delta}$
\end{lemma}

Finally,
\begin{align*}
\eqref{double_sum} &
\leq \tilde{B} \sum_{j_{k_2}=\frac{j}{\tilde{B}}}^j \left[ (\alpha_{\mu_{\tilde{k}},\delta})^x + \tilde{C}_{\lambda,\mu_{k_2},\mu_{\tilde{k}}} \frac{\exp(-j_{k_2}d_{\lambda,\mu_{k_2},\mu_{\tilde{k}}})}{x^\lambda} \right] \\
&\leq  j \tilde{B} (\alpha_{\mu_{\tilde{k}},\delta})^x +  C_{\lambda,\mu_{k_2},\mu_{\tilde{k}}}\frac{\exp(-jd_{\lambda,\mu_{k_2},\mu_{\tilde{k}}}/\tilde{B})}{x^\lambda},
\end{align*}
which concludes the proof.

\end{document}